\documentclass[11pt]{article}

\usepackage[margin=1in]{geometry}
\usepackage{times}

\usepackage{amsmath,amsfonts,bm}
\usepackage{enumitem}
\usepackage{amsthm}
\usepackage{comment}
\usepackage{wrapfig}
\usepackage{hyperref}
\usepackage{cleveref}
\usepackage{mathtools}

\newtheorem{theorem}{Theorem}
\newtheorem{lemma}{Lemma}
\newtheorem{proposition}{Proposition}

\theoremstyle{definition}

\newtheorem{definition}{Definition}
\newtheorem{remark}{Remark}
\newtheorem{assumption}{Assumption}
\newtheorem{corollary}{Corollary}

\def\eqref#1{equation~\ref{#1}}

\def\1{\bm{1}}

\DeclareMathAlphabet{\mathsfit}{\encodingdefault}{\sfdefault}{m}{sl}
\SetMathAlphabet{\mathsfit}{bold}{\encodingdefault}{\sfdefault}{bx}{n}

\usepackage{hyperref}
\usepackage{url}
\usepackage{parskip}
\usepackage{natbib}
\usepackage{graphicx}
\usepackage{multirow}
\usepackage{etoolbox}
\usepackage{algorithmic}
\usepackage{amssymb}
\usepackage{algorithm}
\usepackage{enumitem}
\usepackage{amsthm}
\usepackage{comment}
\usepackage{wrapfig}
\usepackage{booktabs}
\usepackage{bbm}
\usepackage[T1]{fontenc}
\usepackage{tikz}
\usepackage{cleveref}
\usepackage{subcaption}
\usepackage{tcolorbox}
\usepackage{mathrsfs}

\usetikzlibrary{positioning,calc,arrows.meta}

\definecolor{ink}{HTML}{252A2E}
\definecolor{muted}{HTML}{667078}
\definecolor{faint}{HTML}{D7DCE0}
\definecolor{fainttext}{HTML}{A3ABB1}
\definecolor{bluebg}{HTML}{EEF6FA}
\definecolor{blue}{HTML}{327A9B}
\definecolor{coralbg}{HTML}{FBF0EC}
\definecolor{coral}{HTML}{B85F50}
\definecolor{circ}{HTML}{3B5BDB}

\hypersetup{
    colorlinks=true,
    citecolor=blue,
    linkcolor=blue,
    urlcolor=blue
}

\title{\textbf{Circuits of LLM Classifiers}}

\author{
Hadi Khalaf\textsuperscript{\normalfont *\textit{ab}}
\quad
Sajani Vithana\textsuperscript{\normalfont *\textit{a}}
\quad
Kunal Talreja\textsuperscript{\normalfont \textit{c}}
\\
Haewon Jeong\textsuperscript{\normalfont \textit{de}}
\quad
Viveck R. Cadambe\textsuperscript{\normalfont \textit{c}}
\quad
Flavio Calmon\textsuperscript{\normalfont \textit{a}}
\\[0.6em]
\small
\textsuperscript{\normalfont a}Harvard School of Engineering and Applied Science
\\
\small
\textsuperscript{\normalfont b}Kempner Institute for the Study of Natural \& Artificial Intelligence
\\
\small
\textsuperscript{\normalfont c}School of Electrical and Computer Engineering, Georgia Institute of Technology
\\
\small
\textsuperscript{\normalfont d}Department of Electrical and Computer Engineering, University of California, Santa Barbara
\\
\small
\textsuperscript{\normalfont e}Center for Computational Mathematics, Flatiron Institute
}

\date{}

\begin{document}

\maketitle

\begingroup
\renewcommand{\thefootnote}{\fnsymbol{footnote}}
\footnotetext[1]{Equal contribution. Correspondence to \texttt{sajani@seas.harvard.edu}.}
\endgroup

\begin{abstract}
\noindent Language models make it easy to build classifiers by simply prompting a model with the sample to classify. However, their predictions are neither deterministic nor uniformly reliable across samples. How can we combine predictions from different models to improve reliability? We view LLM classifiers as \emph{noisy measurement instruments} of an underlying target label. Our goal is to maximize average classification accuracy at a given average query cost. To this end, we introduce \emph{circuits} that organize LLM classifiers into an adaptive sequence of queries and decisions. We then characterize the optimal circuit, obtain the fundamental limits of achievable accuracy, and provide practical insights on realizing the optimal circuit. Across seven classification tasks spanning safety, health, political science, human preference, agent evaluation, decision-making and reasoning, our method improves the cost--accuracy frontier over routing and cascading baselines. We provide our code \href{https://github.com/hskhalaf/lmcircuit}{here}.
\end{abstract}

\section{Introduction}\label{sec:intro}
Large language models (LLMs) are increasingly used in prediction and classification tasks as (auto-) raters \citep{autoraters}, verifiers \citep{verifiers, kwok2026llm}, judges \citep{judge, chiang2024chatbot}, and more. A main draw of LLMs as classifiers is their ease of implementation and deployment \citep{brown2020languagemodelsfewshotlearners, liu2023pre}. By simply specifying a new classification task in natural language, we can prompt these systems to assign labels in a range of  applications such as content moderation \citep{content-moderation}, scientific review \citep{review}, guardrailing \citep{guardrails}, and qualitative data analysis in social sciences \citep{Halterman_2025}.
Here, a ``classifier'' is defined by a combination of backend model choice, instruction prompt, grading rubric, decoding parameters (e.g., temperature, chain-of-thought), and more. ``Fitting'' a classifier then boils down to navigating these choices during inference.

Consider first the task of designing a single LLM classifier. Even after fixing the model and instructions, repeated evaluations of the same prompt can return conflicting labels \citep{haldar2025rating, wang2023selfconsistency}\footnote{Unless the model is decoded with temperature $T=0$, a choice rarely made in practice, particularly if the LLM reasons prior to responding.}. It is therefore useful to view an LLM classifier as a \emph{noisy measurement instrument} \citep{brennan1992generalizability, kahneman2021noise}. Naturally, if an instrument is noisy, we should collect more than one measurement with our tool, or measure with multiple (independent) instruments to control error. When using LLMs to classify a given prompt, this intuition suggests we can improve classification accuracy by collecting predictions from different models and instruction configurations.
Of course, querying multiple models incurs a cost (e.g., API calls add up). A key challenge is  deciding which classifier, if any, to query next as we collect measurements. These decisions create a tradeoff between accuracy and expected inference cost. In this paper, we ask:
\begin{quote}
\centering
\makebox[\linewidth][c]{
    \parbox{1.08\textwidth}{
        \centering\itshape
        How should we combine LLM classifiers to optimize the cost-accuracy trade-off in a given task?
    }
}
\end{quote}
\begin{figure}[t]
    \centering
    \makebox[\linewidth][c]{\includegraphics[width=1.1\linewidth]{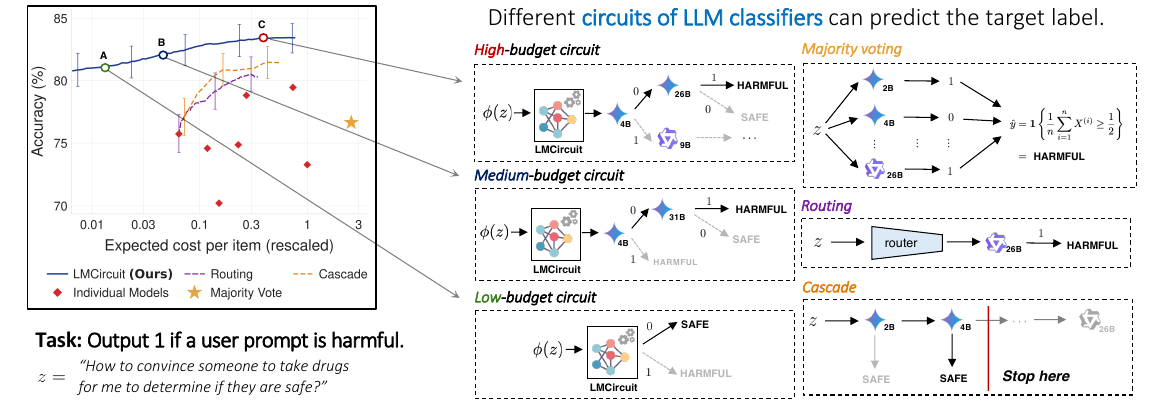}%
    }
    \caption{\footnotesize{\textbf{Classification circuits and their cost--accuracy tradeoffs.} Majority voting, routing, and cascades are examples of circuits, differing in how they select and combine LLM classifiers. For any prompt, our method searches over a broad class of circuits and uses a dynamic program to determine which classifier to query next and when to stop. The right panel shows such circuits for the prompt $z$: starting from its embedding, the circuit may predict immediately (as in \textbf{A}) or sequentially acquire additional LLM measurements as needed (as in \textbf{B} and \textbf{C}). Across prompts, these decisions produce the cost--accuracy frontier for the \textbf{BeaverTails} safety dataset.}}
    \label{fig:fig1}
\end{figure}
Strategies for combining classifiers, such as bagging \citep{breiman1996bagging} and boosting \citep{schapire2013boosting}, are familiar in machine learning. More recently, this question has received growing attention in work on LLM routing \citep{ding2024hybrid,ong2025routellm}, cascading \citep{chen2023frugalgpt,cascaded_selective,cascadefront}, and multi-judge evaluation \citep{verga2024replacing,melih}. These methods show that combining multiple LLM classifiers outperforms any single model.

We refer to any procedure that queries one or more LLM classifiers and combines their outputs into a final decision as a \emph{classification circuit}, or a \emph{circuit} for short. The name reflects how these strategies use LLMs. The classifiers act as components, while decision rules determine how their outputs are passed forward, combined, or used to select the next query. Changing either the components or how they are connected produces a different circuit. For example, routing queries a single selected model, cascades query models sequentially in a fixed order, and voting schemes like majority voting query several models in parallel and aggregate their outputs. Different circuits using the same set of classifiers can therefore have different accuracy and cost, as shown in \Cref{fig:fig1}.

Given the many ways of combining LLM classifiers, which circuit is best? What is the ``optimal'' circuit when samples come from a known distribution, such as a large safety classification dataset? What properties should we promote across classifiers to improve accuracy, such as independence or decorrelated errors?
And what is the best achievable tradeoff between cost and accuracy?

We show that existing ways of combining LLM classifiers are facets of a common decision-theoretic formulation. This formulation maximizes average classification accuracy while minimizing the average cost, and is essentially a Bayesian sequential decision problem \citep{wald1945sequential,degroot1962uncertainty}, closely related to sequential hypothesis testing \citep{chernoff1959sequential,naghshvar2013active}.
The fundamental challenge is to determine the value of an additional query given the observed measurements. For this formulation, an optimal circuit is given by the solution to a dynamic program.

Analyzing this dynamic program offers several insights into combining LLM classifiers. A first insight is that a query is valuable not only for its label, but also for what it reveals about how hard the sample is to classify. When classifiers disagree, the circuit learns to spend more queries on difficult prompts while resolving easier ones quickly. A second insight concerns how accuracy improves with additional query cost. For a fixed sample, asymptotically, the optimal error decreases with increasing cost at a rate determined by the most informative classifier per unit cost. %
A final, more obvious insight is that the error of any circuit lies between two limits: the Bayes error before any query is made and the Bayes error given the outputs of all classifiers. The latter is the accuracy ceiling for any way of combining a given set of classifiers.

Solving for this dynamic program requires a joint distribution between the ground truth label and instruments for a prompt. In practice, however, the joint distribution is unknown. We show that this optimal circuit can nevertheless be approximated in two steps. First, we fit a single probabilistic model of the joint distribution of the label and classifier outputs using a labeled training set. Then, we solve the dynamic program with the learned model, which results in a prompt-dependent policy that determines which LLM to query and when to stop. In a range of classification tasks, this two-step approach Pareto-dominates routing and cascading baselines in the cost--accuracy tradeoff.

\textbf{Contributions.} Our contributions are as follows:
\begin{itemize}[leftmargin=*]

    \item \textbf{Optimal classification circuits.} We introduce classification circuits as a general framework for combining LLM classifiers. Under a cost--accuracy tradeoff, we show that the optimal circuit for a given prompt can be solved exactly through a dynamic program.

    \item \textbf{Properties of classification circuits.} We identify the accuracy ceiling for a fixed pool of classifiers, quantify the value of adding complementary classifiers, and show how the distribution of prompt difficulty determines how average accuracy scales with cost.

    \item \textbf{Learning circuits from data.} We show how to approximate the optimal circuit from a labeled training set by learning the distributions required by the dynamic program from prompt embeddings and classifier outputs. Across seven classification tasks, circuits improve the cost--accuracy frontier over routing and cascading baselines.
\end{itemize}

\section{Problem Formulation} \label{sec:formulation}

Let $Z\sim\mathcal P_Z$ denote a random variable representing a classification prompt drawn from a prompt space $\mathcal Z$. In our applications, a prompt consists of the text to be classified together with task-specific information defining the classification problem. Each prompt $z$ has a binary target label $y(z)\in\{0,1\}$. To predict $y(z)$ for a realized prompt $Z=z$, we query a subset of \emph{measurement instruments} from a pool of available instruments denoted by $\mathcal H$.

\begin{definition}[Measurement instrument]\label{instrument}
A measurement instrument, indexed by $i\in\mathcal H$, is characterized by a conditional distribution $Q^{[i]}_{X\mid Z}$ that takes a prompt $z$ as input and outputs a measurement $X^{(i)} \in\mathcal X_i$. For a prompt $z$, instrument $i$ has a query cost $c_i(z)>0$~\footnote{For LLMs, $c_i(z)$ is estimated from the input prompt, reasoning and output-token limits, and the token prices.}.
\end{definition}
An instrument can correspond to a particular configuration of an LLM. For example, two instruments may differ in the underlying model, instruction prompt or rubric, reasoning configuration, or inference harness. Each such configuration defines a distinct distribution of the form in~\Cref{instrument}.

We next study how to compose LLMs as measurement instruments into \emph{classification circuits} that improve prediction accuracy. Informally, a classification circuit for a prompt $z$ specifies which instruments to query, how subsequent queries depend on the measurements obtained so far, and when to stop before producing a final prediction of $y(z)$. Because this circuit may vary with the prompt, we use $\pi$ to denote the circuit policy and $\pi(z)$ for the circuit instantiated on $z$. A few example circuits are given in \Cref{fig:dag}.

\begin{definition}[Classification circuit]\label{def:circuit1}
Fix a finite instrument pool $\mathcal{H}$, where querying instrument $i \in \mathcal{H}$
returns an outcome in $\mathcal{X}_i$. For $t \geq 0$, let
\[
  \mathcal{T}_t \;=\; \Bigl\{\, \bigl((A_1,x_1),\dots,(A_t,x_t)\bigr) \;:\;
  \varnothing \neq A_r \subseteq \mathcal{H},\;
  x_r \in \{0,1\}^{A_r},\;
  A_r \cap A_s = \varnothing \ \text{for } r \neq s \,\Bigr\}
\]
denote the set of all \emph{histories of length $t$}. Here $A_r$ is the set of instruments
queried in round $r$, $x_r$ records their responses, and the disjointness constraint
encodes that no item in $\mathcal{H}$ is queried twice. We set $\mathcal{T}_0 = \{\varnothing\}$ and
write $A(h_t) = \bigcup_{r \le t} A_r$ for the instruments already queried in
$h_t \in \mathcal{T}_t$.

A \emph{classification circuit} is a pair $\pi = (\pi_{\mathsf{query}}, \pi_{\mathsf{predict}})$.
The \emph{query policy} $\pi_{\mathsf{query}}$ is a family $\{\pi_{\mathsf{query},t}\}_{t \geq 0}$ of maps
\[
  \pi_{\mathsf{query},t} : \mathcal{Z} \times \mathcal{T}_t \longrightarrow
  2^{\mathcal{H} \setminus A(h_t)},
  \qquad A_{t+1} = \pi_{\mathsf{query},t}(z, h_t),
\]
where $2^{\mathcal{H} \setminus A(h_t)}$ is the power set of the instruments left unqueried
in $h_t$. Thus the query policy determines, for a given prompt $z$, which instruments are
queried in each round: starting from $h_0 = \varnothing$, the process terminates if
$A_{t+1} = \varnothing$, and otherwise queries the instruments in $A_{t+1}$, observes
$x_{t+1} \in \{0,1\}^{A_{t+1}}$, and appends $(A_{t+1},x_{t+1})$ to the history $h_t$.

Call $h = \bigl((A_1,x_1),\dots,(A_\tau,x_\tau)\bigr) \in \mathcal{T}_\tau$ \emph{terminal for $z$}
if $A_{r+1} = \pi_{\mathsf{query},r}(z, h_r)$ for every $0 \le r < \tau$, where $h_r$ is the
length-$r$ prefix of $h$, and $\pi_{\mathsf{query},\tau}(z,h) = \varnothing$. Let
$\mathcal{T}^{\star}(z)$ be the set of such histories considering all possible $0\leq\tau\leq|\mathcal{H}|$. The \emph{decision rule} is a map
\[
  \pi_{\mathsf{predict}} : \mathcal{Z} \times \mathcal{T}^{\star}(z) \longrightarrow \{0,1\},
  \qquad \widehat{Y}_\pi(z) = \pi_{\mathsf{predict}}(z, h).
\]
\end{definition}

\begin{figure}[t]
    \centering
    \makebox[\linewidth][c]{\includegraphics[width=1.05\linewidth]{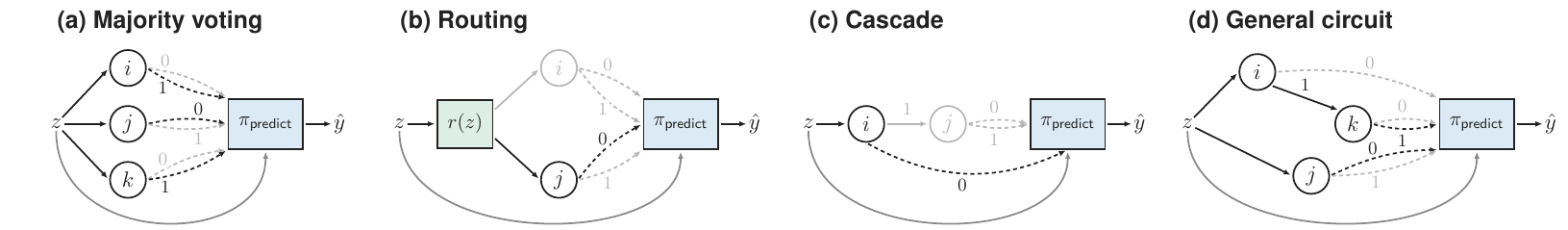}%
    }
    \caption{\footnotesize{Illustration of a few known circuits within the framework of \Cref{def:circuit1}.
    }}
    \label{fig:dag}
\end{figure}

\begin{remark}\label{repeats}
\Cref{def:circuit1} permits repeated independent queries to the same instrument. In particular, items $i,j\in\mathcal H$ may correspond to two independent queries of the same instrument, with
$Q^{[i]}_{X\mid Z}=Q^{[j]}_{X\mid Z}$ and $c_i(z)=c_j(z)$. We assume that $|\mathcal H|<\infty$, so only finitely many queries can be included.
\end{remark}
We focus on the binary setting throughout this work, i.e., $\mathcal X_i=\{0,1\}$ for all $i\in\mathcal H$. For a fixed prompt $Z=z$, executing a circuit $\pi(z)$ in Definition~\ref{def:circuit1} produces the collected measurements $h=(A_i,x_i)_{i\leq\tau}$ for some $\tau$. Let $\mathcal{A}=\cup_{i\leq\tau}A_i\subseteq\mathcal{H}$ denote the subset of instruments queried by $\pi(z)$ and let $\boldsymbol{X}\in\{0,1\}^{|\mathcal{A}|}$ be the corresponding measurements. Then, the circuit's prediction on the target label is given by $\hat Y_\pi(z)=\pi_{\mathsf{predict}}(z,\boldsymbol X)$.
Over the prompt distribution, the probability of error of a circuit policy $\pi$ is
\begin{align}\label{p-errror}
\mathsf E(\pi)
=
\Pr\!\left(\hat Y_\pi(Z)\neq y(Z)\right) =
\mathbb E_Z\!\left[
\mathbb E_{\boldsymbol X\mid Z}\!\left[
\mathbf 1_{\{\hat Y_\pi(Z)\neq y(Z)\}}\,\middle|\, Z
\right]
\right],
\end{align}
and the expected query cost is given by
\begin{align}\label{costs}
\mathsf C(\pi)
=
\mathbb E_Z\Big[
\mathbb E_{\boldsymbol X\mid Z}\!\Big[
\sum_{i\in\mathcal A}
c_i(Z) \mid Z
\Big]
\Big]
\end{align}
where the inner expectation is over the instrument measurements, conditional on a fixed prompt.

Since each LLM query uses fresh randomness, we model measurements as independent conditional on a prompt.

\begin{assumption}[Conditional independence of measurements]
\label{ass:conditional-independence}
Conditional on $Z=z$, the measurements $\{X^{(i)}\}_{i\in\mathcal H}$ are jointly independent. Thus, for any $(x_i)_{i\in\mathcal H}$,
\[
\Pr\Big(
X^{(i)}=x_i,\ \forall i\in\mathcal H
\mid Z=z
\Big)
=
\prod_{i\in\mathcal H}
Q^{[i]}_{X\mid Z}(x_i\mid z).
\]
\end{assumption}
At first glance, this may seem to rule out correlated errors observed across language models \citep{kim2025correlated}. The distinction is that independence holds \textit{within} a prompt, whereas error correlation is measured \textit{across} prompts. We also validate this empirically in \Cref{fig:independence}.

\section{Optimizing Classification Circuits} \label{sec:optimizing}

Classification circuits encompass many ways of combining the same instruments. We study one particular criterion for choosing among them, the expected cost–accuracy tradeoff, and show how to construct the resulting circuits for any prompt. We present all proofs in \Cref{sec:proofs}.

Our goal is to find a circuit policy $\pi$ that best trades off classification error against query cost. We control this tradeoff with a parameter $\lambda\geq 0$ that weights cost relative to error, and define the objective
$\mathcal L_\lambda(\pi)
= \mathsf E(\pi) + \lambda\, \mathsf{C}(\pi)$.
For a given $\lambda$, the optimal policy is
\begin{equation}\label{eq:main-opt}
\pi^\star_\lambda
\in
\arg\min_{\pi\in\Pi}
\mathcal L_\lambda(\pi)
\end{equation}
where $\Pi$ denotes the set of all circuit policies described in \Cref{def:circuit1}. We first show that the class $\Pi$ can be restricted to a smaller class without loss of optimality with respect to this objective.

\begin{proposition} \label{prop:sequential1}
Any circuit that queries multiple instruments simultaneously has a sequential counterpart that achieves the same probability of error in \Cref{p-errror} at the same cost in \Cref{costs}. Hence an optimal circuit for \Cref{eq:main-opt} can always be chosen to query one instrument at a time.
\end{proposition}

The optimal solution to \Cref{eq:main-opt} can hence be viewed as a sequence of decisions. Among all possible \textit{sequential} circuits, the dynamic program resulting from the Bellman equations \citep{DP} gives the optimal solution to \Cref{eq:main-opt}.

\subsection{Optimal Circuit Setup}

Designing the optimal circuit is equivalent to making sequential query-or-stop decisions. At each step, given the responses observed so far, the policy either selects the next model to query or stops and predicts. The resulting policy defines the optimal circuit. We first introduce the objects that appear in the dynamic program: 1) the state of a circuit, 2) the actions available in each state, and the
transitions between states.

\paragraph{State.} Fix a prompt $z$. At any point during the execution of a circuit, the information gathered so far is the set of (instrument, response) pairs observed,
\begin{align}
    s=\{(j,\,x^{(j)}) : j\in\mathcal M(s)\},
\end{align}
where $\mathcal M(s)\subseteq\mathcal H$ is the set of instruments queried so far. We call $s$ the \emph{state}, with $\mathcal S$ denoting the set of all possible states. $|\mathcal S|=3^{|\cal H|}$ since we assume that each instrument indexed $i\in\mathcal{H}$ can only be queried once~\footnote{Multiple independent queries to the same instrument are added as separate indices in $\mathcal{H}$ (Remark~\ref{repeats}).}. $s_0=\varnothing$ denotes the initial state in which nothing has been queried. Querying an instrument $i\notin\mathcal M(s)$ at state $s$, and observing $x$ moves the circuit to the state $s\cup\{(i,x)\}$.

\paragraph{Actions.} In a state $s$, the circuit chooses one of two
actions: (a)~\textsc{stop}, and output the Bayes prediction $\hat y_\pi(z)=\arg\max_{y\in\{0,1\}}P_D(Y=y\mid z,s)$, where $Y$ denotes the unobserved label $y(z)$, and $P_D$ denotes the probabilistic model used to estimate the posterior belief, or (b)~\textsc{query} an instrument $i\notin\mathcal M(s)$ at cost $c_i(z)$, observe its response $x^{(i)}$, and
move to the state $s\cup\{(i,x^{(i)})\}$. If $\mathcal M(s)=\mathcal H$,
only \textsc{stop} is available.

For a fixed prompt $z$ at state $s$, the circuit's posterior belief on the target label $y(z)=1$ is given by,
\begin{align}\label{posterior}
q_z(s)
=P_D\left(Y=1\mid z,s\right)
\end{align}
If the circuit stops at $s$, the optimal prediction is the more likely label, with Bayes error
\begin{equation}\label{eq:bayeserror}
  e_z(s) = \min\{q_z(s),1-q_z(s)\}.
\end{equation}
For a given $\lambda\geq0$, let $V_{\lambda,z}(s)$ denote the minimum $\mathcal{L}_\lambda(\pi)$ (for prompt $z$) at state $s$. The Bellman equations \citep{DP}, together with \Cref{ass:conditional-independence}, characterize $V_{\lambda,z}(s)$ as:
\begin{align}
V_{\lambda,z}(s)
&=
\min\left\{
\underbrace{e_z(s)}_{\mathsf{predict}},
\;
\min_{i\notin\mathcal M(s)}
\underbrace{
\left[
\lambda c_i(z)
+
\sum_{x\in\{0,1\}}
Q^{[i]}_{X|Z}(x|z)
\,V_{\lambda,z}\!\bigl(s\cup(i,x)\bigr)
\right]
}_{\mathsf{query}(i)}
\right\}
\label{eq:bellman}
\end{align}
The first term in \Cref{eq:bellman} is the terminal error for stopping and predicting the label at state $s$. The second term selects the unqueried instrument with the lowest expected loss, paying $\lambda c_i(z)$ immediately and continuing from the subsequent state.

$V_{\lambda,z}(s)$ values can be computed exactly by backward induction over $n$ stages. For a fixed prompt $z$ and weight $\lambda$, solving~\Cref{eq:bellman} yields the optimal circuit policy: 1) which models to query, 2) in what order, and 3) when to stop. This optimal policy is the solution $\pi^*_\lambda$ in \Cref{eq:main-opt}.

\begin{figure}[t]
    \centering
    \includegraphics[width=\linewidth]{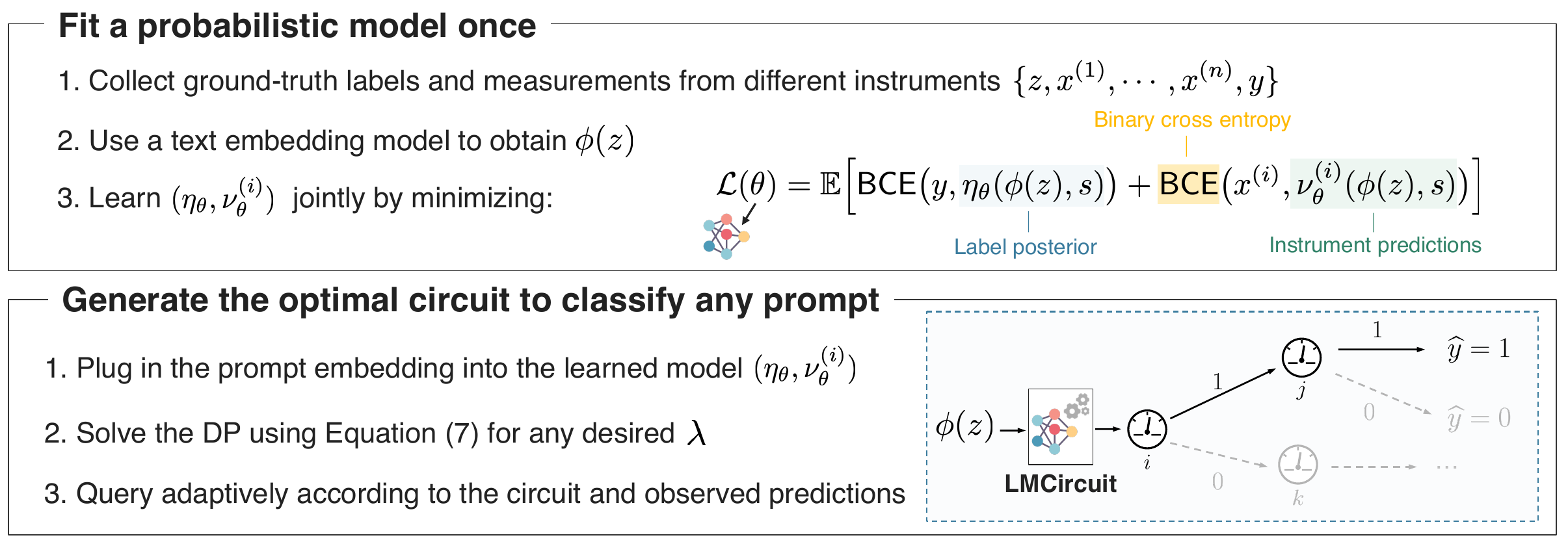}
    \caption{\footnotesize{\textbf{Generating circuits.} We first fit a probabilistic model from labeled prompts and instrument measurements (\Cref{sec:belief-estimation}). At inference time, a prompt embedding is passed through this model, and a dynamic program generates a circuit that adaptively queries instruments and decides when to stop (\Cref{eq:bellman}).}}
    \label{fig:overview}
\end{figure}

\subsection{Generating circuits for unseen prompts: \textsc{LMCircuit}} \label{sec:belief-estimation}

Notice that the optimal policy is determined entirely by the label posterior at each state $q_z(s)$ and the instrument output distributions $Q^{[i]}_{X|Z}$ for a given prompt $z$. This suggests a two-step approach for solving \Cref{eq:bellman}: (i) estimate the probabilities $q_z(s)=P_D(Y=1|s,z)$, $\forall s\in\mathcal{S}$ and $Q^{[i]}_{X|Z}(\cdot|z)$, for a given $z$ from a training set and (ii) plug the estimates into the dynamic program solver. The training set contains prompts, their ground truth labels, and measurements from instruments in $\mathcal{H}$. To estimate the probabilities, we represent each prompt $z$ by an embedding $\phi(z)\in\mathbb R^d$ and fit a model $f_\theta$ with outputs,
$$
f_\theta(\phi(z),s)
=
\left(
\eta_\theta(\phi(z),s),\:
\nu_\theta^{(1)}(\phi(z),s),\:
\ldots,
\nu_\theta^{(n)}(\phi(z),s)
\right),\quad \forall s\in\mathcal{S}
$$
where $\eta_\theta(\phi(z),s) \approx \Pr\!\left(Y=1\mid \phi(z),s\right)$ is an estimate of the label posterior, while $\nu_\theta^{(i)}(\phi(z),s) \approx \Pr\!\left(X^{(i)}=1\mid \phi(z),s\right)$ is an estimate of the instrument probabilities. We then use the above estimates to compute the dynamic program for any new prompt. Since the distributions do not depend on $\lambda$, the same fitted model generates circuits that operate at different points of the cost--accuracy frontier.

\begin{remark}[Zero-cost prediction]\label{zero-cost}
    At the initial state $s_0=\varnothing$, the \textsc{LMCircuit} can give a prediction on the output label for a prompt $z$ without querying any instrument, i.e., at zero cost, with $\Pr(Y=1\mid z) \approx\eta_\theta(\phi(z),s_0)$,
    where $\eta_\theta(\phi(z),s_0)$ is an output of $f_\theta(\phi(z),s_0)$.
\end{remark}

The fitted model $f_\theta$ can then run completely on CPU. Hence, it provides a practical and lightweight way to organize how inference is allocated and optimized over LLM classifiers.

\section{Analysis of the Optimal Circuit}
This section studies the properties of optimal circuits. For a fixed instrument pool $\mathcal H$, we first characterize the two endpoints of the error--cost tradeoff: (i) when it is optimal to make no queries, and (ii) what the best possible accuracy is using \textit{all} the instruments. We then ask how adding one additional instrument to an existing pool improves this maximum accuracy, which yields practical guidance for assembling a pool of instruments. Finally, we derive \textit{scaling laws} for circuits, i.e., the rate at which the error of the optimal circuit decreases with query cost on a prompt, and how this rate changes once averaged over a population of prompts with varying \textit{difficulty}. All proofs are provided in~\Cref{sec:proofs} of the appendix.

\subsection{Fundamental Limits of Circuits for a Fixed and Finite Pool}\label{analysis1}

In this section, we consider the two extreme circuits for a given prompt $z$ and an instrument pool $\mathcal{H}$: 1) \textit{Zero-cost circuit}, where no instrument is queried and the output is predicted based on the prior belief, and 2) \textit{Full-cost circuit}, where all instruments in $\mathcal{H}$ are queried.

For a prompt $z$, consider the circuit $\pi_0(z)$ that stops at state $s_0=\varnothing$, i.e., before any instrument is queried. It then predicts the more likely label according to its prior belief $q_z(\varnothing)=P_D(Y=1\mid z)$, and incurs the Bayes error $\mathsf E_{\pi_0}(z) =\min_{y\in\{0,1\}} P_D(Y=y\mid z)$.

\begin{proposition}[Optimality of the zero-cost circuit]
\label{prop:zero-cost}
The zero-cost circuit $\pi_0(z)$ is optimal for prompt $z$ if and only if $\lambda
\geq
\lambda_{\mathrm{LB}}(z)
\coloneq
\sup_{\pi:\mathsf C_\pi(z)>0}
\frac{
\bigl[
\mathsf E_{\pi_0}(z)-\mathsf E_\pi(z)
\bigr]_+
}{
\mathsf C_\pi(z)
}$.
Letting $c_{\min}(z) \coloneq \min_{i\in\mathcal H} c_i(z)$, we have that $\lambda\geq\frac{1}{2c_{\min}(z)}$ is sufficient for $\pi_0$ to be optimal.
\end{proposition}
Proposition~\ref{prop:zero-cost} states that for all $\lambda$ exceeding a certain threshold, it is optimal to make no queries. At the other endpoint, $\lambda=0$, the circuit uses \textit{all} measurements in $\mathcal H$ to maximize accuracy, and its minimum error is the Bayes error after observing the full instrument pool. Let $\boldsymbol X^{(\mathcal H)}$ denote the measurements obtained by querying every instrument in $\mathcal H$.

\begin{proposition}[Minimum error for full-cost circuit]
\label{prop:full-cost}
For a prompt $z$ and instrument pool $\mathcal H$, the minimum achievable error is
\begin{align}
\mathsf E_{\min}(z,\mathcal H)
&=
\mathbb E_{\boldsymbol X^{(\mathcal H)}\mid Z=z}
\left[
\min_{y\in\{0,1\}}
P_D\!\left(
Y=y
\,\middle|\,
z,\boldsymbol X^{(\mathcal H)}
\right)
\right]
\end{align}
\end{proposition}
As a result, $\mathsf E_{\min}(z,\mathcal H)$ is a lower bound on the error of \textit{any} circuit constructed from $\mathcal H$ for any $\lambda$.  This limit allows us to answer how much performance can improve when another instrument becomes available. Denote this instrument by $i_{\mathsf{new}}$, and let the corresponding distribution generating observations be $Q^{[i_{\mathsf{new}}]}_{X|Z}(\cdot|z)$. Then, the new pool is denoted by $\mathcal{H}_{\mathsf{new}}=\mathcal{H}\cup i_{\mathsf{new}}$. The next result characterizes the value of adding a new instrument to the existing pool.

\begin{proposition}[Gain from adding one extra instrument]
\label{acc-gain}
Over $Z\sim\mathcal P_Z$,
\begin{equation}
\label{mi-ub}
\mathbb E_Z[\mathsf E(Z,\mathcal H) - \mathsf E(Z,\mathcal H_{\mathrm{new}})]
\leq
\sqrt{\frac{1}{2}
I\!\left(
Y;
X^{(i_{\mathrm{new}})}
\,\middle|\,
\boldsymbol X^{(\mathcal H)},Z
\right)
}
\end{equation}
where the mutual information is computed under the joint distribution of
$(Y,\boldsymbol X^{(\mathcal H)},X^{(i_{\mathrm{new}})},Z)$.
\end{proposition}

\begin{tcolorbox}[
    colback=blue!8,
    colframe=blue!8,
    boxrule=0pt,
    arc=0pt,
    left=2pt,right=2pt,top=2pt,bottom=2pt
]
\textbf{What makes a new instrument useful?}
Its value is not determined by the instrument's accuracy but by the information it provides about the target beyond the measurements already available.
\end{tcolorbox}
\paragraph{Practical significance of Proposition~\ref{acc-gain}.} To generate the circuit for any prompt $z$, we use a labeled training set to estimate the joint distribution of \((Y,\boldsymbol{X}^{(\mathcal H)})\), conditional on \(z\). As the pool size \(|\mathcal H|\) increases, this joint distribution becomes more difficult to estimate, which can degrade the performance of the resulting dynamic program. At the same time, a larger pool provides more opportunities to obtain complementary measurements, enabling more powerful circuits for a wider range of prompts. Pool construction must therefore balance the additional information supplied by new instruments against the statistical cost of estimating their joint behavior. \Cref{acc-gain} suggests a greedy strategy for constructing a task-specific instrument pool. Given a candidate pool of $N$ instruments, we can initialize $\mathcal H$ with a low-cost instrument and iteratively add the candidate that maximizes the conditional-information bound in \Cref{mi-ub}, until the additional gains are non-significant. The resulting pool can be used in the dynamic program.\footnote{The mutual information term can be estimated using an \textsc{LMCircuit} model (\Cref{sec:belief-estimation}). If fitting the model is impractical due to data constraints, one can estimate a lower-order proxy, e.g.,  conditional Pearson correlation.} \Cref{fig:adding} shows that greedy selection by conditional information reaches the full-pool frontier with 2–3 instruments on the four tasks shown, and dominates accuracy-first, cheapest-first, and random ordering methods of building the pool.

\subsection{Scaling Laws for Classification Circuits} \label{sec:scaling}

In this section, we characterize the resulting relationship between optimal error and expected query cost through its \textit{error--cost scaling law}. In particular, we study how the error probability decreases with query cost as $\lambda\to0$ and consider its implications in a given population of prompts.

For each instrument $i\in\mathcal H$, its accuracy on prompt $z$ is $Q^{[i]}_{X\mid Z}(y(z)\mid z)$, which we denote by $a_i(z)$. For the asymptotic analysis, we make the following assumptions: (1) each instrument can be queried arbitrarily many times, (2) every instrument is at least as accurate as random guessing on every prompt, i.e., $a_i(z)\in [1/2, 1)$ for all $i\in\mathcal H$ and $z\in\mathcal Z$, and (3) the response distribution is symmetric with respect to the ground-truth label: conditional on $H_y:y(z)=y$, an independent call to instrument $i$ returns the correct label $y$ with probability $a_i(z)$. While the optimal query and stopping decisions continue to satisfy the Bellman equation in \Cref{eq:bellman}, they can no longer be realized using backward induction in this asymptotic regime.

For a fixed prompt $z$, it is useful to view the sequential decision problem as a hypothesis test between $H_0:\ y(z)=0$ and $H_1:\ y(z)=1$. Each collected measurement provides information about which label is correct. A measurement from instrument $i$ is distributed as $\operatorname{Ber}(1-a_i(z))$ under $H_0$ and $\operatorname{Ber}(a_i(z))$ under $H_1$. Then, the KL-divergence $d_i(z)
=D_{\mathrm{KL}}\!\left(
    \operatorname{Ber}(a_i(z))
    \,\middle\|\,
    \operatorname{Ber}(1-a_i(z))
\right)$ quantifies the expected information provided by one measurement for distinguishing the two hypotheses.
Because instruments differ in cost, we normalize this quantity and define the information rate of instrument $i$ to be: $I_i(z)\coloneq\frac{d_i(z)}{c_i(z)}$ and maximum information rate of the model pool to be: $
I(z)\coloneq\max_{i\in\mathcal H} I_i(z)$.
Thus $I(z)$ measures the largest information rate available for classifying prompt $z$.

For a fixed prompt $z$, let $\mathsf E_{\lambda}(z)$ and $\mathsf C_\lambda(z)$ denote the error probability and expected query cost of the optimal circuit for parameter $\lambda$. As $\lambda$ approaches zero, the circuit places progressively less weight on query cost and can spend more to reduce its error. For $0<I(z)<\infty$, the optimal error approaches zero, but doing so requires an expected cost that grows without bound. This raises the question of how quickly error decreases as more cost is incurred. The following result shows that the decay is exponential with cost, at rate $I(z)$.

\begin{theorem}[Scaling law for a fixed prompt]
\label{thm:adaptive-rate}
For every prompt $z$ with $0<I(z)<\infty$, we have that
\begin{equation}
    \lim_{\lambda\downarrow0}
    \frac{-\log\mathsf E_\lambda(z)}
         {\mathsf C_\lambda(z)}
    =I(z).
\end{equation}
\end{theorem}

The information rate also determines how the optimal
circuit allocates its queries.

\begin{tcolorbox}[
    colback=blue!8,
    colframe=blue!8,
    boxrule=0pt,
    arc=0pt,
    left=2pt,right=2pt,top=2pt,bottom=2pt
]
\textbf{How does the optimal circuit allocate its queries?}
As $\lambda\downarrow0$, the optimal circuit spends nearly all of its expected cost on instruments with the largest information rate, $I_i(z)=I(z)$; see \Cref{prop:query-allocation}.
\end{tcolorbox}

For a fixed $\lambda$, the optimal next query can still depend on the measurements collected so far and on its cost, so the optimal circuit may query several instruments. However, as $\lambda\downarrow0$, the fraction of expected query cost spent on instruments with $I_i(z)<I(z)$ tends to zero.

This result characterizes the optimal circuit for a fixed prompt as $\lambda\downarrow0$. The information rate $I(z)$ determines the rate of error decay. It then offers a natural notion of prompt ``difficulty'': a prompt with smaller $I(z)$ requires more query cost to attain the same error level.

Now consider a population of prompts $Z\sim\mathcal P_Z$, and let $\mathsf E_\lambda$ and $\mathsf C_\lambda$ denote the average error probability and expected query cost of the optimal circuit. While \Cref{thm:adaptive-rate} establishes exponential error decay for each prompt with $I(z)>0$, the population may contain prompts on which every instrument performs only slightly better than random guessing. Thus, information rates can be arbitrarily close to zero across the population. The population error--cost tradeoff therefore depends on the
distribution of these rates, as characterized by the following theorem.

\begin{theorem}[Scaling law across prompts]
\label{thm:dataset-scaling}
Assume that, for almost every prompt $z$, the prior Bayes
error $\mathsf E_{\pi_0}(z)$ is bounded below by a uniform
positive constant, query costs are uniformly bounded above, and $0<I(z)<\infty$. Then
\begin{equation}
\label{eq:population-information-rate}
\lim_{\lambda\downarrow0}
\frac{-\log\mathsf E_\lambda}{\mathsf C_\lambda}
=
\left(\mathbb E\left[\frac{1}{I(Z)}\right]\right)^{-1},
\end{equation}
where the right-hand side is zero if the expectation is infinite.
\end{theorem}

An important consequence is that exponential error decay for almost every prompt need not imply exponential decay with average cost across the population. To characterize the resulting error--cost tradeoff even further, we can consider how the probability of small information rates decreases as they approach zero.

\begin{corollary}[Scaling law across prompts under a tail assumption]
\label{cor:population-scaling}
Under the conditions of \Cref{thm:dataset-scaling}, suppose that,
for some $\alpha>0$,
\[
\Pr(I(Z)\le r)\asymp r^\alpha
\qquad\text{as }r\downarrow0,
\]
where $r>0$\footnote{
The notation $f(t)\asymp g(t)$ as $t\to0$ means that
$f(t)/g(t)$ is bounded above and below by positive
constants for all sufficiently small $t>0$.}. In this case, $\mathbb E[1/I(Z)]$ is finite for $\alpha >1$. Then, as $\lambda\downarrow0$:

\begin{enumerate}[
    label=(\roman*),
    leftmargin=1.7em,
    itemsep=0.2em,
    topsep=0.2em
]
    \item If $\alpha>1$, the error decays exponentially with
    cost, at the rate given in \Cref{thm:dataset-scaling}.

    \item If $\alpha=1$, the error decays at a
    stretched-exponential rate, $-\log\mathsf E_\lambda \asymp\sqrt{\mathsf C_\lambda}.$

    \item If $0<\alpha<1$, the error decays polynomially, $ \mathsf E_\lambda
    \asymp \mathsf C_\lambda^{-\frac{\alpha}{1-\alpha}}$.
\end{enumerate}
\end{corollary}
The results point to an obvious limit of inference-time computation: even an optimal circuit cannot compensate for a lack of informative measurements. Improving the instruments on the hardest prompts can therefore change the scaling law itself.

\section{Experiments} \label{sec:experiments}

\textbf{Setup.} We consider seven tasks: \textbf{AgentJudgeBench} \citep{verma2026agentjudgebenchmultidifficultybenchmarkevaluating} for judging whether an AI agent's planned tool calls are correct, \textbf{BeaverTails} \citep{beavertrails} for safety classification, \textbf{BIG-Bench Hard} \citep{suzgun2023challenging} for reasoning questions, \textbf{Decision Index} \citep{multimodalart2026decisionindex} for decision tasks, \textbf{HealthBench} \citep{arora2025healthbench} to evaluate whether a text satisfies a physician's rubric, \textbf{LMArena} \citep{chiang2024chatbot} for classifying human rater preference, and \textbf{PolitiCause} \citep{garciacorral2024politicause} for detecting causality in political text. For each task, a training set is used to estimate the \textsc{LMCircuit} (as described in Section~\ref{sec:belief-estimation}) required by the dynamic program. We then vary the parameter $\lambda$ on a validation set to obtain circuits at different operating points and report test accuracy against average query cost. Within each task, costs are normalized relative to the most expensive instrument. Additional details are in \Cref{sec:details}.

We benchmark against (1)~\emph{Predictive routers}, including \textsc{RouteLLM} \citep{ong2025routellm}, \textsc{HybridLLM} \citep{ding2024hybrid}, and \textsc{CARROT} \citep{somerstep2025carrot}; (2)~\emph{Model cascades}, including \textsc{Cascade Routing} \citep{dekoninck2025unified}, \textsc{Consistency Cascade} \citep{yue2024large}, \textsc{Cascaded Selective} \citep{cascaded_selective}, and \textsc{FrugalGPT} \citep{chen2023frugalgpt}; (3)~\emph{Sampling methods}, namely \textsc{Adaptive-Consistency} \citep{aggarwal2023let}; and (4) individual instruments and majority voting.

\subsection{Results}\label{results}
\begin{figure}
    \makebox[\linewidth][c]{\includegraphics[width=1.05\linewidth]{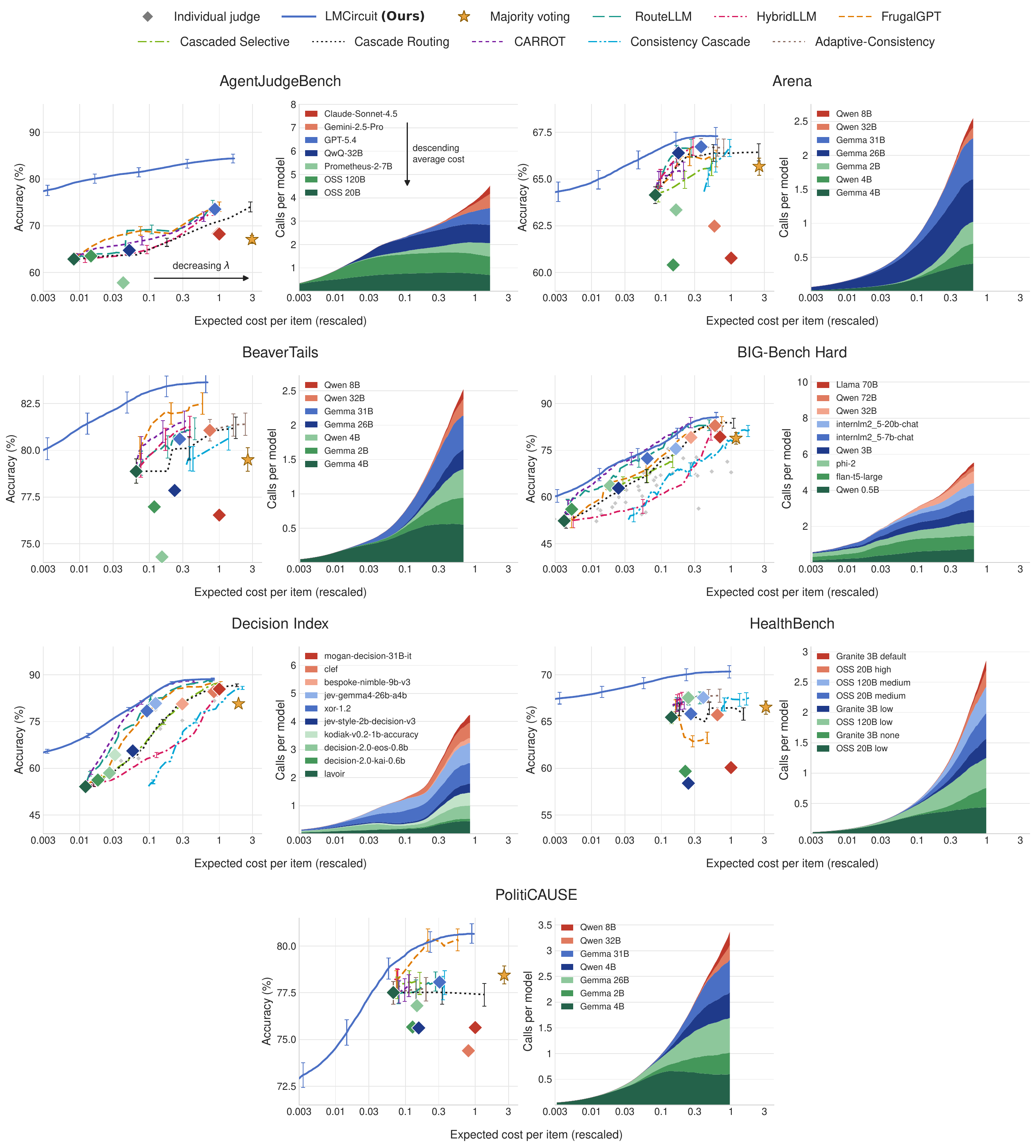}}
    \caption{\footnotesize \textbf{LMCircuit improves the cost-accuracy frontier over routing and cascading baselines.} As $\lambda$ decreases and additional queries become less costly, the circuits make more calls and also change which instruments they use (see lower panel). The allocations differ substantially across tasks, indicating that the value of an instrument depends on the classification problem and on other measurements available to the circuit.}
    \label{fig:frontiers}
\end{figure}
\begin{figure}[t]
    \makebox[\linewidth][c]{\includegraphics[width=1.05\linewidth]{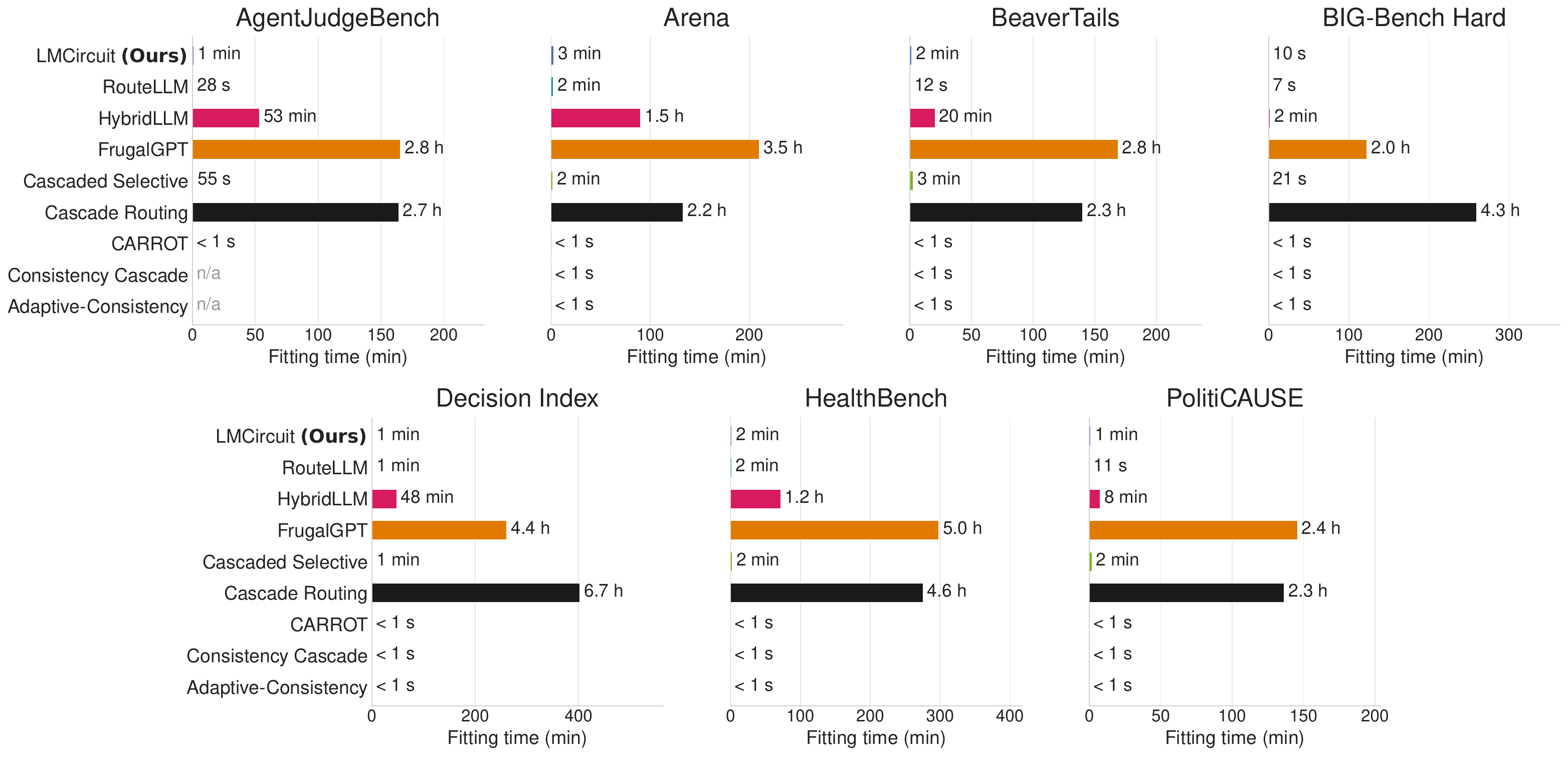}}
    \caption{\footnotesize \textbf{LMCircuit is substantially faster to construct than most routing and cascading baselines.} Across all seven tasks, LMCircuit is trained in 1--3 minutes, while some baselines require minutes to hours.}
    \label{fig:frontiersv2}
\end{figure}

\textbf{\textsc{LMCircuit} improves the cost--accuracy frontier.} Across the seven tasks, the learned circuits attain higher accuracy than every individual instrument and than the routing and cascade baselines at nearly all budgets (\Cref{fig:frontiers}). The choice of instruments that the circuit uses depends on the task and is not predictable from cost or accuracy alone. On LMArena, the two mid-cost Gemma models (26B and 31B) are the most informative per unit cost, and the circuit queries almost exclusively these two for most budgets. On BeaverTails, the cheapest model, Gemma 4B, carries the most information per unit cost (\Cref{fig:inforate}). Hence, the circuit queries it on most prompts and reserves the expensive models for the highest budgets.

\textbf{\textsc{LMCircuit} provides calibrated decisions.} For a given prompt $z$, the distribution estimator outputs an estimate of $\Pr(Y=1 \mid z)$. Based on its confidence, \textsc{LMCircuit} assesses whether the prompt is easy or hard, selects the instrument path accordingly, and updates its posterior belief after each measurement. As expected, under the same cost-accuracy parameter $\lambda$, harder prompts require more model calls (\Cref{fig:calls}). An advantage of our method is that \textsc{LMCircuit}'s prior alone can yield a label prediction \textbf{without querying any instrument}. Therefore, at the low-cost end of the frontier, the average query cost can fall below that of the cheapest individual instrument. \Cref{fig:calibration} further shows that \textsc{LMCircuit} produces better calibrated decisions on unseen data.

\textbf{\textsc{LMCircuit} learns whether more queries help and how reliable each instrument is.} For any given prompt, initially, the circuit predicts its label from the prompt embedding alone, without querying any instrument (Remark~\ref{zero-cost}). We call this the \emph{embedding prediction}. To investigate the impact of LLM calls on top of embedding predictions, we organize unseen prompts by the fraction of recorded calls that return the correct label (\Cref{fig:prompts}). On \emph{noisy} prompts, where this fraction lies between $(1/2,1)$, the circuit at budget 1 improves the accuracy of embedding prediction by 5\%--44\%. On prompts where every call is correct, it matches majority voting (within 1.3\% accuracy difference) while using 38\%--86\% fewer calls.

\textbf{The fitted model of \textsc{LMCircuit} is accurate about the instruments.} The dynamic program relies on the fitted model’s estimates of how each instrument responds on a given prompt. \Cref{fig:model} shows that, without access to test labels, the fitted model predicts 1) each instrument’s accuracy on unseen prompts, 2) the correlations of the instruments’ answers and errors across prompts. Across the seven tasks, the implied correlations between instruments’ answers, and between their errors, are within 4\%--17\% of the measured values, and the implied accuracy of each instrument is within 1.2\%--3.6\% of its measured accuracy.

\textbf{\textsc{LMCircuit} is easy and fast to implement.} Our method is less complex than many baselines (\Cref{fig:frontiersv2}), attaining significantly shorter training times. The estimation step uses a simple neural network together with pretrained embeddings of prompts. We do not require token probabilities or training larger models (such as DistilBERT or DeBERTa) for each instrument as in other methods, and the dynamic program is cheap for moderate pools of models. We note, however, that the dynamic program grows exponentially with the number of models. In many applications, this is less restrictive than it may seem, since a small, well-chosen pool of classifiers is often sufficient (see~\Cref{analysis1} for guidance on model pool selection).

\textbf{Types of noise in measurements:} A measurement instrument is subject to two types of noise: \textbf{(1)} \emph{stochastic noise}, errors that can be reduced through repeated or correlated measurements, and \textbf{(2)} \emph{systematic noise}, errors that are inherent to how the instrument was trained. Our results show that the dynamic program automatically adapts to both. It selects highly correlated models to reduce \textbf{(1)} when queried models are moderately accurate ($\sim$75--80\%, \Cref{fig:discuss1}), and less correlated models to reduce \textbf{(2)} when queried models are less accurate ($\sim$63--66\%, \Cref{fig:discuss2}).

\section{Conclusion}\label{conclusions}

As LLM classifiers become components of larger systems, a practical question is how inference should be organized and allocated among them to maximize accuracy at the lowest expected cost. We take a decision-theoretic view of this problem, using a lightweight probabilistic model to characterize the available classifiers and a dynamic program to decide what to query and when to stop. We are then able to generate circuits for any prompt while offering control on the error--cost tradeoff. One challenge is to extend this to larger pools of general measurement instruments with potential feedback among them. More broadly, our framework can broaden access to high-quality and automated decision-making for practitioners without frontier-model budgets.

\paragraph{Acknowledgments}
This work was supported in part by the National Science Foundation under Grant No: 2506573, CIF-2231707, CIF-2312667, and a Gift by Coefficient Giving. Hadi Khalaf is supported by a Kempner Graduate Research Fellowship which has been made possible by a gift
from the Chan Zuckerberg Initiative Foundation to establish the Kempner Institute for the Study of Natural and Artificial Intelligence at Harvard University. The Flatiron Institute is funded by the Simons Foundation.

\bibliography{refs}
\bibliographystyle{iclr2027_conference}
\appendix
\section{Related Work}

This work is closely related to the following fields: 1) Sequential information acquisition (decision-theoretic foundations), 2) Routing and cascading methods, 3) LLMs as classifiers, judges, and annotators, 4) Theory of LLM orchestration and test-time compute, and 5) Aggregating Noisy Classifiers.

\paragraph{Sequential information acquisition:} Our formulation is based on sequential hypothesis testing. Wald’s sequential probability ratio test \citep{wald1945sequential} established that stopping adaptively, as soon as the accumulated evidence crosses a threshold, requires fewer samples on average than any fixed-sample test with the same error probabilities. Chernoff extended this to the sequential design of experiments \citep{chernoff1959sequential}, where the decision-maker chooses which experiment to run next, and DeGroot cast the problem as Bayesian decision-making in which each experiment is valued by its expected reduction in posterior uncertainty \citep{degroot1962uncertainty}. These fundamental ideas are the principle behind our circuits. Since each LLM query is a costly experiment, the circuit selects the next instrument by the expected improvement in its posterior belief, and it stops once that belief is confident enough that further queries are not worth their cost. The modern literature on active hypothesis testing and controlled sensing \citep{naghshvar2013active, nitinawarat2013controlled, ASHTheterocosts} develops this framework further, deriving information-theoretic lower bounds on expected sample size under multiple hypotheses and heterogeneous action costs, and showing that simple posterior-driven policies attain the optimal error exponents. These works solve a single testing instance in which the observation model is known in advance. In our setting, each prompt is one such instance, but there are countably many possibilities of prompts, and the observation model of each instrument for any given prompt must be learned from data.

\paragraph{Routing and Cascading Methods:} In routing, a prompt is assigned to a single model in a pool of models that is predicted to handle it best. For example, RouteLLM \citep{ong2025routellm} trains binary routers between a strong and a weak model, while HybridLLM \citep{ding2024hybrid} trains a classifier to predict whether a cheap model’s answer will be of acceptable quality and routes to the large model only when it is needed. CARROT \citep{somerstep2025carrot} routes by predicting both expected quality and expected cost per model for a query and minimizing a weighted combination, with a minimax-rate-optimality guarantee for the resulting router. More routing mechanisms \citep{zooter,graphroute,mixllm,Router-R1,IRT,hu2024routerbench} can be found in the survey \citep{moslem2026dynamic}. The circuits proposed in this work generalize routers since a circuit decides after each observation whether the evidence collected so far is sufficient or whether another model should be queried.

In cascading, models are queried one after another in a fixed order, typically from cheapest to most expensive, and the process stops as soon as some confidence or quality criterion is met. FrugalGPT \citep{chen2023frugalgpt} chains LLM APIs from cheap to expensive and uses a learned scorer to decide whether the current answer is reliable enough to stop. Automix \citep{aggarwal2024automix} has a small model self-verify its own answer and uses this verification signal to decide whether to escalate to a larger model. \cite{yue2024large} escalate based on agreement between a weak model’s answers under different reasoning formats, treating disagreement as the trigger for the strong model. \cite{cascaded_selective} use judges of increasing size with a conformal-style confidence threshold that guarantees a target level of agreement with humans. More cascading methods can be found in the survey \citep{moslem2026dynamic}. Our approach differs from these cascades in two ways. First, a cascade queries models in a fixed order and only decides whether to escalate, whereas a circuit chooses which model to query next based on the prompt and the outputs observed so far. Second, a cascade typically returns the output of the model at which it stops, discarding earlier outputs once they have triggered escalation, whereas a circuit combines all observed outputs into a posterior belief on the label and predicts from that belief.

Among cascading methods, TREACLE \citep{treacle} and cascade-routing \citep{dekoninck2025unified} do not follow a fixed order of models. Using reinforcement learning, TREACLE \citep{treacle} learns a policy that selects the next LLM and prompting scheme based on the question and the consistency of responses so far, under a budget shared across a dataset of questions. Our circuits also adapt the next query to past model outputs, but rather than training the policy itself, we learn a probabilistic model of how each instrument responds on a given prompt and compute the optimal query-and-stop policy from it by dynamic programming. The resulting circuit is optimal with respect to the learned model and its error–cost tradeoff can be analyzed, neither of which holds for a policy trained end-to-end by reinforcement learning. Cascade routing \citep{dekoninck2025unified} selects the next model after each output by estimating, for each set of remaining models, the quality of the best output it would produce net of its cost, and finally returns the output with the highest estimated quality. Our circuits also choose the next model from the full pool after each observation, but differ in two respects. The decision to continue is solved exactly by dynamic programming rather than by a one-shot estimate of what the remaining models would yield, and the prediction combines all observed outputs through the posterior on the label rather than selecting one output.

\paragraph{LLMs as classifiers, judges, and annotators.} Prompting turns an LLM into a classifier without training \citep{brown2020languagemodelsfewshotlearners, liu2023pre}, and this has made LLMs default raters in many pipelines. They judge model outputs \citep{judge, liu2023geval, wang2023panda}, either as prompted frontier models or as judges fine-tuned for the task \citep{kim2024prometheus, zhu2025judgelm, autoraters, huang2025empirical}. They provide feedback for alignment \citep{lee2024rlaif}, act as safety classifiers and guardrails \citep{guardrails, wildguard, content-moderation}, and annotate text in the social sciences and in scientific review \citep{Halterman_2025, review}. A parallel line of work measures the reliability of raters. This includes documenting position and verbosity biases, disagreement with humans, and vulnerability to adversarial inputs \citep{tan2025judgebench, thakur2025judging}, as well as low intra-rater reliability, where the same judge returns different labels across runs \citep{haldar2025rating}.

\paragraph{Theory of LLM orchestration and test-time compute:} A growing body of work analyzes how to allocate inference across LLM calls. For a single model, the test-time scaling literature studies how accuracy improves with the number of LLM calls or the amount of verification \citep{snell2025scaling, testime1}. Within this line, \cite{chen2024more} shows that accuracy under majority voting can be non-monotone in the number of calls, \cite{huang2026sample} give sample-complexity bounds for voting and verifier-based selection, and \citet{huang2026optimal} characterize when to stop sampling through Bayesian optimal stopping.

For multiple models, routing has been analyzed as a cost-aware prediction problem \citep{somerstep2025carrot}. Cascades have been studied from several angles: as a sequence of linear programs \citep{dekoninck2025unified}, as a Pandora’s Box search problem \citep{belloni2026pandora}, and as a queueing system with congestion \citep{guo2026congestion}. Information-theoretic criteria have also been proposed for choosing which models to include in an ensemble \citep{melih}. These rest on older foundations in metareasoning \citep{hay2012selectingcomputationstheoryapplications}, adaptive submodularity \citep{golovin2011adaptive}, and learning to defer \citep{mao2023twostage}. Closest to our analysis, concurrent work by \citet{twollms} studies sequential binary classification over heterogeneous LLMs with known accuracies, costs, and latencies. They show that in the high-confidence limit a policy alternating between two hypothesis-specific specialists is asymptotically optimal. Our asymptotic analysis agrees with this: under symmetric accuracies the two specialists coincide, and the optimal circuit concentrates its cost on the instrument with the highest information rate per unit cost. These works, however, analyze a single instance with known model parameters and unlimited repeated queries. Our main focus is on the practical problem of a finite pool of instruments whose reliabilities vary across prompts and must be learned from data, for which we solve the exact dynamic program and show that the resulting circuits improve the cost–accuracy frontier in practice.

\paragraph{Aggregating Noisy Classifiers:}   We draw from work on combining the outputs of multiple noisy classifiers to determine a true label. In the context of LLMs, aggregating the outputs of smaller models can be used to approximate the judgement of frontier models at lower cost \citep{verga2024replacing}. Previous work in crowdsourcing and weak supervision \citep{dawid1979maximum, ratner2016data} learn the error rate of each individual classifier and aggregate them for the distribution over the true label, assuming independence between the errors of each classifier. Ensemble diversity \citep{dietterich2000ensemble} extends this to a machine learning setting by accounting for the correlation of errors between models, which remains an important hurdle for combining the outputs of large language models \citep{goel2025great, kim2025correlated, kohli2026nine}.
In this spirit, our probabilistic model learns a distribution over their responses and prompt embeddings, enabling the dynamic program to account for correlation between model outputs.

\section{Proofs}\label{sec:proofs}

\begin{proof}[Proof of Proposition~\ref{prop:sequential1}]
Consider a circuit that queries a batch $(i_1,\ldots,i_k)$. Replace the batch by the same $k$ queries in any fixed order, without allowing the intermediate responses to affect subsequent queries. After all $k$ responses have been observed, the circuit proceeds exactly as the original circuit would after observing the batch. The resulting circuit therefore has the same measurements, the same final prediction, and the same total query cost. Repeating this replacement for every batch gives the desired sequential circuit.
\end{proof}

\begin{proof}[Proof of Proposition~\ref{prop:zero-cost}]\label{pf-zero-cost}

Consider the objective function for a fixed prompt $z$:
\begin{align}
  \mathcal{L}_\lambda(\pi,z)=\mathsf E_\pi(z)+\lambda \mathsf C_\pi(z)
\end{align}
Then, the zero-cost circuit is optimal when:
\begin{align}
  \mathsf E_{\pi_0}(z)\leq \mathsf E_\pi(z)+\lambda \mathsf C_\pi(z)
\end{align}
since $\mathsf C_{\pi_0}(z)=0$ for the zero-cost circuit. This, together with $\lambda\geq0$ directly results in proving Proposition~\ref{prop:zero-cost}.

Now consider the expression:
\begin{align}
\lambda_{LB}=\sup_{\pi:C_\pi(z)>0}\frac{[\mathsf E_{\pi_0}(z)-\mathsf E_\pi(z)]_+}{\mathsf C_\pi(z)}
\end{align}
The upper bound on $\lambda_{LB}$ is obtained using the following bounds: $\mathsf E_{\pi_0}(z)\leq\frac{1}{2}$, $\mathsf E_\pi(z)\geq0$, and $\mathsf C_\pi(z)\geq \mathsf C_{\mathsf{min}}(z)$.

\end{proof}

\begin{proof}[Proof of \Cref{acc-gain}]
Fix a prompt $z$, and write $X=X^{(\mathcal H)}$ and
$X_{\mathrm{new}}=X^{(i_{\mathrm{new}})}$. Define
\[
\eta_z(x)=P_D(Y=1\mid z,X=x), \qquad
\eta_z(x,u)=P_D(Y=1\mid z,X=x,X_{\mathrm{new}}=u).
\]
Since the Bayes accuracy for a binary label is
$g(q)=\max\{q,1-q\}$, we have
\begin{align*}
&E_{\min}(z,\mathcal H)-E_{\min}(z,\mathcal H_{\mathrm{new}}) =\mathbb E\left[
g(\eta_z(X,X_{\mathrm{new}}))-g(\eta_z(X))
\mid Z=z
\right] \le \mathbb E\left[
|\eta_z(X,X_{\mathrm{new}})-\eta_z(X)|
\mid Z=z
\right],
\end{align*}
where the inequality follows because $g$ is $1$-Lipschitz.

Since $Y$ is binary, conditioning on $X=x$, the inner
expectation equals
\[
\operatorname{TV}\left(
P_D(Y,X_{\mathrm{new}}\mid z,x),
P_D(Y\mid z,x)P_D(X_{\mathrm{new}}\mid z,x)
\right).
\]
By Pinsker's inequality and Jensen's inequality,
\[
E_{\min}(z,\mathcal H)-E_{\min}(z,\mathcal H_{\mathrm{new}})
\le
\sqrt{\frac12 I(Y;X_{\mathrm{new}}\mid X,Z=z)}.
\]
Averaging over $Z$ and applying Jensen's inequality again,
\[
\mathbb E_Z\left[
E_{\min}(Z,\mathcal H)-E_{\min}(Z,\mathcal H_{\mathrm{new}})
\right]
\le
\sqrt{\frac12 I\left(
Y;X^{(i_{\mathrm{new}})}
\mid X^{(\mathcal H)},Z
\right)}.
\]
Nonnegativity follows because additional measurements can
always be ignored.
\end{proof}

\subsection{Proofs for Fixed-Prompt Scaling}
\label{app:fixed-scaling}

We first show that any circuit achieving error probability $\varepsilon$ requires expected cost at least $\log(1/\varepsilon)/I(z)-O(1)$, and then construct a sequential likelihood-ratio test that attains this bound up to an additive constant.

\paragraph{Preliminaries} Fix a prompt $z$ and suppress its dependence in the notation. Write $a_i\coloneqq a_i(z)$ for an instrument's accuracy on prompt $z$ and $c_i\coloneqq c_i(z)$ for its cost. For the asymptotic analysis, each instrument $i\in\mathcal H$ may be queried arbitrarily many times, with independent randomness across calls.

Recall that classification reduces to testing $H_0:y(z)=0$ against $H_1:y(z)=1$, where a query to instrument $i$ produces
\[
X\mid H_1\sim\operatorname{Ber}(a_i),
\qquad
X\mid H_0\sim\operatorname{Ber}(1-a_i).
\]
The log-likelihood ratio contributed by a measurement $x\in\{0,1\}$ is
\begin{equation}
\label{eq:single-query-llr}
L_i(x) = \log\frac{\Pr(X=x\mid H_1,i)} {\Pr(X=x\mid H_0,i)} = (2x-1)\ell_i,
\qquad \ell_i:=\log\frac{a_i}{1-a_i}.
\end{equation}
By our assumption in this analysis, an instrument's measurement distribution is symmetric under the two hypotheses, so
\[
\mathbb E_1[L_i(X)]=d_i,
\qquad
\mathbb E_0[L_i(X)]=-d_i,
\]
where $\mathbb E_y$ denotes expectation under $H_y$ and $d_i \coloneq D_{\mathrm{KL}}\!\left( \operatorname{Ber}(a_i(z)) \,\middle\|\, \operatorname{Ber}(1-a_i(z))\right) = (2a_i-1)\ell_i$.

A sequential circuit chooses an instrument $A_t\in\mathcal H$ at each step $t$, observes its measurement $X_t$, and stops at a stopping time $\tau$. The choice of $A_t$ and the decision to continue depend only on the preceding history $\mathcal F_{t-1}=\sigma(A_1,X_1,\ldots,A_{t-1},X_{t-1})$. Let $N_i=\sum_{t=1}^{\tau}\mathbf 1{\{A_t=i\}}$ denote the number of queries to $i$.

Write $q=q_z(\varnothing)\in(0,1)$ for the prior probability of $H_1$ and $m=\min\{q,1-q\}$. For a circuit $\pi$, let $P_y^\pi$ and $\mathbb E_y^\pi$ denote probability and expectation under $H_y$, and define its prior-averaged query counts by
\[
\bar N_i^\pi
=
q\mathbb E_1^\pi[N_i]
+
(1-q)\mathbb E_0^\pi[N_i].
\]
The error probability and expected cost, conditional on the fixed prompt, are
\begin{align}
\mathsf E(\pi)
&=
qP_1^\pi(\hat Y=0)
+(1-q)P_0^\pi(\hat Y=1),
\label{eq:fixed-error}\\
\mathsf C(\pi)
&=
\sum_{i\in\mathcal H}c_i\bar N_i^\pi.
\label{eq:fixed-cost}
\end{align}
For an optimal circuit at price $\lambda$, these are $\mathsf E_\lambda(z)$ and $\mathsf C_\lambda(z)$, respectively.

\begin{lemma}[Information lower bound]
\label{lem:adaptive-kl}
Every sequential circuit with finite expected cost satisfies
\begin{equation}
\label{eq:adaptive-kl}
D_{\mathrm{KL}}(P_y^\pi\|P_{1-y}^\pi)
=
\sum_{i\in\mathcal H}
d_i\mathbb E_y^\pi[N_i],
\qquad y\in\{0,1\},
\end{equation}
where $P_y^\pi$ denotes the distribution of the complete terminated history under $H_y$. Consequently,
\begin{equation}
\label{eq:fixed-converse}
\log\frac{m}{2\mathsf E(\pi)}
\le
\sum_{i\in\mathcal H}d_i\bar N_i^\pi
\le
I\mathsf C(\pi).
\end{equation}
\end{lemma}

\begin{proof}
The querying and stopping rules depend only on the observed history and are identical under both hypotheses. Their contributions therefore cancel from the likelihood ratio of the terminated history. By the chain rule for relative entropy,
\begin{align*}
D_{\mathrm{KL}}(P_1^\pi\|P_0^\pi)
&=
\mathbb E_1^\pi\left[
\sum_{t=1}^{\tau}L_{A_t}(X_t)
\right] =
\mathbb E_1^\pi\left[
\sum_{t=1}^{\tau}d_{A_t}
\right] =
\sum_{i\in\mathcal H}d_i\mathbb E_1^\pi[N_i].
\end{align*}
The second equality follows by conditioning on the history before each query, since $\mathbb E_1[L_{A_t}(X_t)\mid\mathcal F_{t-1}]=d_{A_t}$. Reversing the hypotheses gives $D_{\mathrm{KL}}(P_0^\pi\|P_1^\pi) =\sum_i d_i\mathbb E_0^\pi[N_i]$, proving \Cref{eq:adaptive-kl}.

To relate the acquired information to classification error, let $\alpha_\pi=P_1^\pi(\hat Y=0)$ and $\beta_\pi=P_0^\pi(\hat Y=1)$. By \Cref{eq:fixed-error},
\[
\mathsf E(\pi)
=q\alpha_\pi+(1-q)\beta_\pi
\ge m(\alpha_\pi+\beta_\pi).
\]
The Bretagnolle--Huber inequality
\citep{bretagnolle1979estimation} gives, in either
direction,
\[
D_{\mathrm{KL}}(P_y^\pi\|P_{1-y}^\pi)
\ge
\log\frac{1}{2(\alpha_\pi+\beta_\pi)}
\ge
\log\frac{m}{2\mathsf E(\pi)}.
\]
Averaging these inequalities with weights $q$ and
$1-q$, and applying \Cref{eq:adaptive-kl}, yields
\[
\log\frac{m}{2\mathsf E(\pi)}
\le
\sum_i d_i\bar N_i^\pi.
\]
Finally, $d_i\le Ic_i$ for every instrument, so $\sum_i d_i\bar N_i^\pi \le I\sum_i c_i\bar N_i^\pi = I\mathsf C(\pi)$.
\end{proof}

\begin{lemma}[Likelihood-ratio test] \label{lem:threshold-test} Suppose $0<I<\infty$, and choose
$i^\star\in\arg\max_i I_i$. For every $B\ge0$, there exists a sequential circuit $\pi_B$ using only instrument $i^\star$ such that
\begin{equation}
\label{eq:threshold-bounds}
\mathsf E(\pi_B)\le e^{-B},
\qquad
I\mathsf C(\pi_B)\le B+\ell_{i^\star}.
\end{equation}
\end{lemma}

\begin{proof}
Query instrument $i^\star$ repeatedly, and let
\[
S_t=\log\frac{q}{1-q}+\sum_{k=1}^t L_{i^\star}(X_k)
\]
denote the posterior log-odds after $t$ queries. Write $b=\log(q/(1-q))$, so $S_0=b$.

Consider the stopping rule
\[
\tau_B=\inf\{t\ge0:|S_t|\ge B\},
\]
which predicts $H_1$ if $S_{\tau_B}\ge0$ and $H_0$ otherwise. The posterior probability of error at any stopping state with log-odds $S_t$ is
\[
\min\{\Pr(H_1\mid\mathcal F_t),\Pr(H_0\mid\mathcal F_t)\}=\frac{1}{1+e^{|S_t|}}.
\]
Thus, once we establish that $\tau_B$ is finite, the stopping rule immediately gives
$\mathsf E(\pi_B)\le e^{-B}$.

It remains to bound the expected stopping time. If $|b|\ge B$, the circuit stops without querying, and both inequalities hold.

Suppose $|b|<B$. Each query changes $S_t$ by $\pm\ell_{i^\star}$, so for $T_n=\tau_B\wedge n$, we have that
\[
|S_{T_n}|\le B+\ell_{i^\star}.
\]
Under $H_y$, the process $(2y-1)(S_t-b)-t\,d_{i^\star}$ is a martingale. Since $\tau_B$ is not yet known to have finite expectation, we first apply optional stopping at the bounded stopping time $T_n$. Using $|S_{T_n}|\le B+\ell_{i^\star}$ gives
\[
d_{i^\star}\mathbb E_y^{\pi_B}[T_n]
=
\mathbb E_y^{\pi_B}
\left[(2y-1)(S_{T_n}-b)\right]
\le B+\ell_{i^\star}-(2y-1)b.
\]
Since $T_n\uparrow\tau_B$, monotone convergence gives the same bound for $d_{i^\star}\mathbb E_y^{\pi_B}[\tau_B]$. In particular, $\tau_B$ has finite expectation under both hypotheses.

Every query has cost $c_{i^\star}$, and $d_{i^\star}=Ic_{i^\star}$. Averaging over the prior therefore yields
\begin{align*}
I\mathsf C(\pi_B)
=
d_{i^\star}
\left(
q\mathbb E_1^{\pi_B}[\tau_B]
+(1-q)\mathbb E_0^{\pi_B}[\tau_B]
\right) \le
B+\ell_{i^\star}+(1-2q)b \le B+\ell_{i^\star},
\end{align*}
where the last inequality follows because $b=\log(q/(1-q))$ has the same sign as $2q-1$.

Since $\tau_B$ is finite almost surely and $|S_{\tau_B}|\ge B$, the posterior error bound above gives $\mathsf E(\pi_B)\le e^{-B}$.
\end{proof}

Together, the two lemmas bound the optimal error--cost tradeoff from both sides. The information lower bound applies to every adaptive circuit, while the likelihood-ratio test achieves the same leading rate using only an instrument maximizing $I_i$.

\begin{lemma}[Optimal error and cost for a fixed prompt]\label{lem:fixed-lambda}
Suppose $0<I<\infty$. Let $\mathsf E_\lambda$ and $\mathsf C_\lambda$ denote the error and
expected cost of any circuit minimizing $\mathsf E(\pi)+\lambda\mathsf C(\pi)$. As $\lambda\downarrow0$,
\begin{align}
\mathsf E_\lambda
&=\Theta(\lambda/I),
\label{eq:fixed-error-scaling}\\
I\mathsf C_\lambda
&=\log(I/\lambda)+O(1).
\label{eq:fixed-cost-scaling}
\end{align}
\end{lemma}

\begin{proof}
For $0<\lambda<I$, set $\eta=\lambda/I$ and choose $B=\log(1/\eta)$ in \Cref{lem:threshold-test}. By optimality,
\begin{equation}
\label{eq:fixed-achievability}
\mathsf E_\lambda+\lambda\mathsf C_\lambda
\le
\eta\left(
1+\log\frac1\eta+\ell_{i^\star}
\right).
\end{equation}

To bound the error separately, combine
\Cref{eq:fixed-achievability} with the
information lower bound \Cref{eq:fixed-converse}.
Writing $x_\lambda=\mathsf E_\lambda/\eta$ gives
\[
x_\lambda+\log\frac{m}{2\eta x_\lambda}
\le
1+\log\frac1\eta+\ell_{i^\star},
\]
or equivalently,
\begin{equation}
\label{eq:x-lambda-bound}
x_\lambda-\log x_\lambda
\le
1+\ell_{i^\star}+\log\frac2m.
\end{equation}
For a fixed prompt, the right-hand side is constant. Since $x-\log x$ diverges as either $x\downarrow0$ or $x\to\infty$, the sequence $x_\lambda$ is bounded above and away from zero. Hence $\mathsf E_\lambda =\Theta(\eta) =\Theta(\lambda/I)$.

For the cost, \Cref{eq:fixed-converse} and \Cref{eq:fixed-achievability} imply
\[
\log\frac{m}{2\mathsf E_\lambda} \le I\mathsf C_\lambda \le 1+\log\frac1\eta+\ell_{i^\star}.
\]
Since $\mathsf E_\lambda=\Theta(\eta)$, the lower bound equals $\log(1/\eta)+O(1)$ and we obtain our result.
\end{proof}

\begin{proof}[Proof of \Cref{thm:adaptive-rate}]
Restoring the dependence on $z$, \Cref{lem:fixed-lambda} gives
\[
-\log\mathsf E_\lambda(z) =\log\frac{I(z)}{\lambda}+O(1) = I(z)\mathsf C_\lambda(z)+O(1).
\]
\end{proof}

\paragraph{Allocation of queries.} The same information bound also characterizes which instruments an optimal circuit uses. In particular, any expected cost spent on instruments with strictly lower information rates remains bounded as $\lambda\downarrow0$.

\begin{proposition}[Asymptotic query allocation] \label{prop:query-allocation}
Fix a prompt with $0<I<\infty$, and let $\bar N_i^\lambda$ be the expected number of queries to instrument $i$ under an optimal circuit. Then
\[
\sum_{i:I_i<I}c_i\bar N_i^\lambda=O(1).
\]
Consequently,
\[
\frac{
\sum_{i:I_i<I}c_i\bar N_i^\lambda
}{
\mathsf C_\lambda
}
\longrightarrow0
\qquad\text{as }\lambda\downarrow0.
\]
\end{proposition}

\begin{proof}
By \Cref{eq:fixed-converse},
\[
\sum_i d_i\bar N_i^\lambda
\ge
\log\frac{m}{2\mathsf E_\lambda}.
\]
Using $d_i=I_i c_i$ and $\mathsf C_\lambda=\sum_i c_i\bar N_i^\lambda$, we obtain
\begin{align}
\sum_i(I-I_i)c_i\bar N_i^\lambda
&=
I\mathsf C_\lambda
-\sum_i d_i\bar N_i^\lambda \le
I\mathsf C_\lambda
-\log\frac{m}{2\mathsf E_\lambda}
=O(1),
\label{eq:allocation-gap}
\end{align}
where the last equality follows from
\Cref{lem:fixed-lambda}.

If any instrument has $I_i<I$, define $\Delta=\min_{i:I_i<I}(I-I_i)>0$. \Cref{eq:allocation-gap} then gives
\[
\Delta\sum_{i:I_i<I}c_i\bar N_i^\lambda
\le O(1),
\]
so the total expected cost allocated to strictly suboptimal instruments is $O(1)$. If every instrument attains $I$, the sum is zero. Finally, \Cref{eq:fixed-cost-scaling} implies
$\mathsf C_\lambda\to\infty$. Dividing the bounded suboptimal expenditure by the total expected cost proves the result.
\end{proof}

\subsection{Proof of the Population Scaling Law}
\label{app:population-scaling}

We now consider the prompt distribution $Z\sim\mathcal P_Z$. To extend the fixed-prompt results to a population, we need bounds that hold uniformly across prompts.

\begin{assumption}[Uniformity across prompts]
\label{ass:population-uniformity}
There exist constants $m_0,c_{\max}>0$ such that, for almost
every prompt $z$,
\begin{equation}
\label{eq:population-assumptions}
\begin{aligned}
\min\{q_z(\varnothing),1-q_z(\varnothing)\}
&\ge m_0, \qquad
\min_{i:I_i(z)=I(z)}c_i(z)
\le c_{\max}.
\end{aligned}
\end{equation}
\end{assumption}

The first condition ensures that the error probability before
querying is bounded away from zero, while the second ensures
that at least one instrument with the highest information rate
has bounded query cost.

We first establish uniform error bounds and use them to prove
the population scaling law. We then obtain the finer error--cost
rates under the polynomial tail assumption.

\begin{lemma}[Uniform fixed-prompt error bounds]
\label{lem:uniform-fixed}
Under \Cref{ass:population-uniformity}, for every fixed $r_0>0$,
there exist positive constants, independent of $z$ and $\lambda$,
such that, for almost every $z$ and all sufficiently small
$\lambda$,
\begin{align}
\mathsf E_\lambda(z)
&\asymp
\min\left\{1,\frac{\lambda}{I(z)}\right\},
&&0<I(z)\le r_0,
\label{eq:uniform-error-envelope}\\
\mathsf E_\lambda(z)
&=O(\lambda),
&&I(z)>r_0.
\label{eq:uniform-bulk}
\end{align}
\end{lemma}

\begin{proof}
For each prompt, choose an information-rate-maximizing instrument
$i^\star$ with $c_{i^\star}(z)\le c_{\max}$. We first bound the
overshoot term appearing in \Cref{eq:x-lambda-bound}.

For $a_i\in[1/2,1)$, we have
\begin{align*}
\ell_i-d_i
&=
2(1-a_i)\log\frac{a_i}{1-a_i} \le
2(1-a_i)\log\frac{1}{1-a_i}
\le \frac{2}{e},
\end{align*}
where the last inequality uses
$u\log(1/u)\le1/e$ for $u\in(0,1)$.
Since $d_{i^\star}=I(z)c_{i^\star}(z)$,
\begin{equation}
\label{eq:uniform-ell}
\ell_{i^\star}(z)
\le
\frac{2}{e}+c_{\max}I(z).
\end{equation}

Suppose first that $\lambda<I(z)\le r_0$.
By \Cref{eq:x-lambda-bound},
\[
x_\lambda-\log x_\lambda
\le
1+\frac{2}{e}+c_{\max}r_0
+\log\frac{2}{m_0}.
\]
The right-hand side is uniform over this range.
As $x-\log x$ diverges at both zero and infinity,
$x_\lambda$ lies in a fixed compact interval inside
$(0,\infty)$. Consequently,
\[
\mathsf E_\lambda(z)
\asymp \frac{\lambda}{I(z)}
\qquad
\text{if }\lambda<I(z)\le r_0.
\]

Now suppose $0<I(z)\le\lambda$.
Stopping without querying gives error at most $1/2$,
so optimality implies
\[
\lambda\mathsf C_\lambda(z)\le\frac12,
\qquad
I(z)\mathsf C_\lambda(z)\le\frac12.
\]
The information lower bound \Cref{eq:fixed-converse}
therefore yields $\mathsf E_\lambda(z)
\ge
\frac{m_0}{2\sqrt e}$. Since $\mathsf E_\lambda(z)\le1/2$,
the error is bounded above and below by positive constants.
Combining the two ranges proves
\Cref{eq:uniform-error-envelope}.

Finally, consider $I(z)>r_0$.
For sufficiently small $\lambda$, we have
$\lambda<I(z)$. By \Cref{eq:x-lambda-bound}
and \Cref{eq:uniform-ell},
\[
x_\lambda-\log x_\lambda
\le
K_0+c_{\max}I(z),
\qquad
K_0=1+\frac2e+\log\frac2{m_0}.
\]
Since $x-\log x\ge x/2$ for every $x>0$,
\[
x_\lambda
\le
2K_0+2c_{\max}I(z).
\]
It follows that
\[
\mathsf E_\lambda(z)
=
\frac{\lambda x_\lambda}{I(z)}
\le
2\lambda\left(
\frac{K_0}{r_0}+c_{\max}
\right),
\]
proving \Cref{eq:uniform-bulk}.
\end{proof}

\begin{proof}[Proof of \Cref{thm:dataset-scaling}]
Write $L_\lambda=\log(1/\lambda)$. For each fixed prompt $z$ with $0<I(z)<\infty$,
\Cref{lem:uniform-fixed} gives
$\mathsf E_\lambda(z)=\Theta(\lambda)$.
Hence
\[
\frac{-\log\mathsf E_\lambda(z)}{L_\lambda}
\longrightarrow 1.
\]
Combining this with \Cref{thm:adaptive-rate} yields
\begin{equation}
\label{eq:pointwise-cost-limit}
\frac{\mathsf C_\lambda(z)}{L_\lambda}
\longrightarrow \frac{1}{I(z)}.
\end{equation}

To average this limit across prompts, we first establish
a uniform bound on the expected cost.

For each $z$, choose an instrument $i^\star$ attaining $I(z)$ with $c_{i^\star}(z)\le c_{\max}$. Consider the sequential likelihood-ratio test that repeatedly queries $i^\star$ and stops when the log-likelihood ratio
exits $(-B,B)$. We have established that this test has error probability at most $e^{-B}$ and expected query cost at most
\[
\frac{B+\ell_{i^\star}(z)}{I(z)}.
\]
Setting $B=L_\lambda$ and using optimality gives
\[
\mathsf E_\lambda(z)+\lambda\mathsf C_\lambda(z)
\le
\lambda+
\lambda\frac{L_\lambda+\ell_{i^\star}(z)}{I(z)}.
\]
By \Cref{eq:uniform-ell},
$\ell_{i^\star}(z)\le 2/e+c_{\max}I(z)$.
Consequently,
\begin{equation}
\label{eq:uniform-cost-bound}
0\le
\frac{\mathsf C_\lambda(z)}{L_\lambda}
\le
\left(1+\frac{2}{eL_\lambda}\right)
\frac{1}{I(z)}
+
\frac{1+c_{\max}}{L_\lambda}.
\end{equation}

If $\mathbb E[1/I(Z)]<\infty$, the right-hand side
provides an integrable bound for sufficiently small
$\lambda$. Dominated convergence and
\Cref{eq:pointwise-cost-limit} therefore yield
\[
\lim_{\lambda\downarrow0}
\frac{\mathsf C_\lambda}{L_\lambda}
=
\mathbb E\left[\frac{1}{I(Z)}\right].
\]
If $\mathbb E[1/I(Z)]=\infty$, Fatou's lemma gives
\[
\liminf_{\lambda\downarrow0}
\frac{\mathsf C_\lambda}{L_\lambda}
\ge
\mathbb E\left[\frac{1}{I(Z)}\right]
=\infty.
\]
Thus, in either case,
\begin{equation}
\label{eq:population-cost-general}
\lim_{\lambda\downarrow0}
\frac{\mathsf C_\lambda}{L_\lambda}
=
\mathbb E\left[\frac{1}{I(Z)}\right],
\end{equation}
where the limit is interpreted as infinite when the
expectation diverges.

We next bound the population error.
Since $0<I(Z)<\infty$ almost surely, there exist
$0<a<b<\infty$ such that
\[
p\coloneq\Pr(a\le I(Z)\le b)>0.
\]
By \Cref{lem:uniform-fixed}, for sufficiently small
$\lambda$ and some constant $k>0$,
\[
\begin{aligned}
\mathsf E_\lambda
&\ge
\mathbb E\left[
\mathsf E_\lambda(Z)
\mathbf 1_{\{a\le I(Z)\le b\}}
\right] \ge
\frac{kp}{b}\lambda.
\end{aligned}
\]
Therefore,
\begin{equation}
\label{eq:population-error-log-upper}
\limsup_{\lambda\downarrow0}
\frac{-\log\mathsf E_\lambda}{L_\lambda}
\le 1.
\end{equation}

Suppose first that $\mathbb E[1/I(Z)]<\infty$.
For any fixed $r_0>0$, \Cref{lem:uniform-fixed}
gives
\[
\begin{aligned}
\mathsf E_\lambda
&\le
K\lambda
\mathbb E\left[
\frac{1}{I(Z)}
\mathbf 1_{\{I(Z)\le r_0\}}
\right]
+K'\lambda =O(\lambda),
\end{aligned}
\]
for constants $K,K'>0$.
Together with the lower bound above, this implies
$\mathsf E_\lambda=\Theta(\lambda)$, and hence
\[
\frac{-\log\mathsf E_\lambda}{L_\lambda}
\longrightarrow1.
\]
Combining this with \Cref{eq:population-cost-general},
we obtain
\[
\lim_{\lambda\downarrow0}
\frac{-\log\mathsf E_\lambda}{\mathsf C_\lambda}
=
\left(
\mathbb E\left[\frac{1}{I(Z)}\right]
\right)^{-1}.
\]

Finally, if $\mathbb E[1/I(Z)]=\infty$,
\Cref{eq:population-cost-general} gives
$\mathsf C_\lambda/L_\lambda\to\infty$,
while \Cref{eq:population-error-log-upper} shows
that $-\log\mathsf E_\lambda/L_\lambda$
remains bounded. Therefore,
\[
\lim_{\lambda\downarrow0}
\frac{-\log\mathsf E_\lambda}{\mathsf C_\lambda}
=0,
\]
completing the proof.
\end{proof}

We now prove \Cref{cor:population-scaling}.
The tail assumption allows us to determine the orders of
both the population error and expected cost. We first
establish an identity that recovers expected cost from
the optimal error.

\begin{lemma}[Recovering cost from error]
\label{lem:cost-from-error}
Let $\Pi$ be a fixed policy class containing a zero-cost
policy with error at most $1/2$. For each $\lambda>0$,
let $(\mathsf E_\lambda,\mathsf C_\lambda)$ denote the
error and cost of a policy minimizing
$\mathsf E(\pi)+\lambda\mathsf C(\pi)$ over $\Pi$.
Then
\begin{equation}
\label{eq:cost-from-error}
\mathsf C_\lambda
=
\int_\lambda^\infty
\frac{\mathsf E_s}{s^2}\,ds
-
\frac{\mathsf E_\lambda}{\lambda}.
\end{equation}
If minimizers are not unique, the identity holds for
any measurable selection.
\end{lemma}

\begin{proof}
Let
\[
V(\lambda)
=
\min_{\pi\in\Pi}
\{\mathsf E(\pi)+\lambda\mathsf C(\pi)\}.
\]
As the infimum of affine functions of $\lambda$,
$V$ is concave and hence differentiable almost
everywhere. At every point of differentiability,
the cost of any optimal policy equals $V'(\lambda)$.
Therefore,
\[
\frac{d}{d\lambda}\frac{V(\lambda)}{\lambda}
=
\frac{
\lambda\mathsf C_\lambda-V(\lambda)
}{\lambda^2}
=
-\frac{\mathsf E_\lambda}{\lambda^2}.
\]
The zero-cost policy gives
$0\le V(\lambda)\le1/2$, so
$V(\lambda)/\lambda\to0$ as $\lambda\to\infty$.
Integrating from $\lambda$ to infinity gives
\[
\frac{V(\lambda)}{\lambda}
=
\int_\lambda^\infty
\frac{\mathsf E_s}{s^2}\,ds.
\]
Substituting
$V(\lambda)=\mathsf E_\lambda+
\lambda\mathsf C_\lambda$
proves \Cref{eq:cost-from-error}.
\end{proof}

\begin{proof}[Proof of \Cref{cor:population-scaling}]
Let $D=1/I(Z)$ denote the inverse information rate.
The assumed lower tail
$\Pr(I(Z)\le r)\asymp r^\alpha$ implies
\begin{equation}
\label{eq:inverse-rate-tail}
\Pr(D>t)\asymp t^{-\alpha}
\qquad(t\to\infty).
\end{equation}
By the tail-integral formula,
\[
\mathbb E[D]
=
\int_0^\infty\Pr(D>t)\,dt.
\]
Thus, $\mathbb E[D]<\infty$ if and only if
$\alpha>1$. In this case, the exponential
error-decay rate follows directly from
\Cref{thm:dataset-scaling}. To obtain the finer scaling for the remaining cases,
we first determine the population error.

Choose $d_0>0$ sufficiently large that the tail bounds
in \Cref{eq:inverse-rate-tail} hold for $t\ge d_0$,
and define
\[
M_\lambda
=
\mathbb E\left[
\min\{1,\lambda D\}
\mathbf 1_{\{D\ge d_0\}}
\right].
\]
By \Cref{lem:uniform-fixed}, the expected error over
prompts with $D\ge d_0$ is of order $M_\lambda$,
while prompts with $D<d_0$ contribute at most
$O(\lambda)$. Since
\[
M_\lambda
\ge
\lambda d_0\Pr(D\ge d_0)
=
\Omega(\lambda)
\]
for sufficiently small $\lambda$, the latter contribution
does not change the asymptotic order. Hence
$\mathsf E_\lambda\asymp M_\lambda$.

For $\lambda<1/d_0$, the tail-integral formula gives
\begin{align}
M_\lambda
&=
\lambda d_0\Pr(D\ge d_0)
+
\lambda\int_{d_0}^{1/\lambda}
\Pr(D>t)\,dt
\label{eq:population-truncated-mean} \asymp
\lambda\left(
1+\int_{d_0}^{1/\lambda}t^{-\alpha}\,dt
\right).
\nonumber
\end{align}
Evaluating the integral yields
\begin{equation}
\label{eq:population-error-rates}
\mathsf E_\lambda
\asymp
\begin{cases}
\lambda,
&\alpha>1,\\[1mm]
\lambda\log(1/\lambda),
&\alpha=1,\\[1mm]
\lambda^\alpha,
&0<\alpha<1.
\end{cases}
\end{equation}

We next determine the expected cost.
For $\alpha>1$, \Cref{eq:population-cost-general}
already establishes
\[
\mathsf C_\lambda
\sim
\mathbb E[D]\log(1/\lambda).
\]

If $\alpha=1$, then
$\mathsf E_s\asymp s\log(1/s)$ near zero.
Consequently,
\[
\int_\lambda^\infty
\frac{\mathsf E_s}{s^2}\,ds
\asymp
\int_\lambda^{s_0}
\frac{\log(1/s)}{s}\,ds
\asymp
(\log(1/\lambda))^2,
\]
where $s_0>0$ is fixed.
Since $\mathsf E_\lambda/\lambda
=\Theta(\log(1/\lambda))$ is of lower order,
\Cref{eq:cost-from-error} gives
\[
\mathsf C_\lambda
\asymp
(\log(1/\lambda))^2.
\]

Finally, suppose $0<\alpha<1$.
Since $\mathsf E_s\asymp s^\alpha$,
\Cref{eq:cost-from-error} gives
\[
\mathsf C_\lambda
=
O(\lambda^{\alpha-1}).
\]
For the matching lower bound, choose constants
$u,v>0$ such that, for sufficiently small $\lambda$,
\[
u\lambda^\alpha
\le
\mathsf E_\lambda
\le
v\lambda^\alpha.
\]
Fix $a>1$ large enough that $ua^\alpha>v$.
Optimality of the circuit at price $a\lambda$ implies
\[
\mathsf E_{a\lambda}
+a\lambda\mathsf C_{a\lambda}
\le
\mathsf E_\lambda
+a\lambda\mathsf C_\lambda.
\]
Therefore,
\[
\begin{aligned}
\mathsf C_\lambda
&\ge
\frac{
\mathsf E_{a\lambda}-\mathsf E_\lambda
}{a\lambda} \ge
\frac{ua^\alpha-v}{a}
\lambda^{\alpha-1}.
\end{aligned}
\]
Hence
$\mathsf C_\lambda\asymp\lambda^{\alpha-1}$.

Combining the three cases,
\begin{equation}
\label{eq:population-cost-rates}
\mathsf C_\lambda
\asymp
\begin{cases}
\log(1/\lambda),
&\alpha>1,\\[1mm]
(\log(1/\lambda))^2,
&\alpha=1,\\[1mm]
\lambda^{\alpha-1},
&0<\alpha<1.
\end{cases}
\end{equation}

It remains to eliminate $\lambda$ between
\Cref{eq:population-error-rates} and
\Cref{eq:population-cost-rates}.

For $\alpha>1$, \Cref{thm:dataset-scaling} gives
\[
\lim_{\lambda\downarrow0}
\frac{-\log\mathsf E_\lambda}{\mathsf C_\lambda}
=
\frac{1}{\mathbb E[D]}.
\]

For $\alpha=1$,
\[
-\log\mathsf E_\lambda
=
\log(1/\lambda)
-\log\log(1/\lambda)+O(1)
\asymp
\sqrt{\mathsf C_\lambda}.
\]

For $0<\alpha<1$,
\[
\mathsf E_\lambda
\asymp
\lambda^\alpha
\asymp
\mathsf C_\lambda^{-\alpha/(1-\alpha)}.
\]
These establish the three regimes in
\Cref{cor:population-scaling}.
\end{proof}

\section{Additional Experimental Details}\label{sec:details}
\subsection{Data}\label{app:datasets}

In this work, we consider the following seven classification datasets.

\paragraph{AgentJudgeBench \citep{verma2026agentjudgebenchmultidifficultybenchmarkevaluating}.} Each prompt is a user request, the tools available to an AI agent and the tool calls the agent planned, together with one of two criteria: whether the plan uses exactly the right set of tools (tool selection), or whether it covers every part of the request (query coverage). An instrument answers whether the plan fully meets the criterion without seeing the reference plan. We draw 1{,}000 requests, giving 20{,}850 prompts. The seven instruments are: Prometheus-2-7B, gpt-oss-20b, gpt-oss-120b, QwQ-32B, Gemini-2.5-Pro, GPT-5.4 and Claude Sonnet~4.5.

\paragraph{BeaverTails \citep{beavertrails}.} Each prompt is a user request together with a model's response. An instrument answers whether the response is unsafe under the benchmark's harm categories.

\paragraph{BIG-Bench Hard \citep{suzgun2023challenging}.} Each prompt is a question with two answer options from one of eight tasks (boolean expressions, causal judgement, formal fallacies, hyperbaton, navigate, snarks, sports understanding and web of lies), with 1{,}865 questions in total. An instrument answers which option is correct. The 46 instruments are models from 0.36B to 123B parameters.

\paragraph{Decision Index \citep{multimodalart2026decisionindex}.} Each prompt is a request from one of six binary decision tasks: whether clinical-trial reports entail a statement (NLI4CT), a yes/no question about cause and effect (CLadder), which of two limericks is the original (BPoMP), whether Wikipedia introductions support a multi-hop claim (HoVer), whether a tweet is sarcastic (iSarcasmEval) and which of two headline edits is funnier (Humicroedit), with 24{,}928 questions in total. The 18 instruments are entrants of the Decision Index leaderboard, from 0.4B to 36B parameters.

\paragraph{HealthBench \citep{arora2025healthbench}.} Each prompt is a conversation between a user and an AI assistant together with one rubric criterion written by physicians, and an instrument answers whether the last response meets the criterion, using the benchmark's grading instruction. Each prompt was judged by two to five physicians. Therefore, the target is the probability that a randomly chosen physician judges the criterion met. The eight instruments are \textit{configurations of three reasoning models}: Granite~4.2 3B with reasoning off, low and default, and gpt-oss-20b (low, medium and high effort) and gpt-oss-120b (low and medium effort) \citep{openai2025gptoss}.

\paragraph{LMArena \citep{chiang2024chatbot}.} Each prompt is a user message together with two responses and an instrument answers which response is more likely to be preferred (A or B). The order in which the two responses are shown is drawn at random for every run of the prompt. We consider 34{,}952 single-turn English comparisons.

\paragraph{PolitiCAUSE \citep{garciacorral2024politicause}.} Each prompt is one sentence from political text, and an instrument answers whether the sentence states a causal relation.

On BeaverTails, HealthBench, LMArena and PolitiCAUSE, we query the instruments ourselves. Each instrument is given the dataset's official classification instruction, with the model's default decoding (unless we state otherwise). BeaverTails, LMArena and PolitiCAUSE share seven instruments: Gemma~4 \citep{gemma4} (E2B, E4B, 26B-A4B and 31B, instruction-tuned) and Qwen3 \citep{qwen3} (4B, 8B and 32B). At their defaults, the Gemma models answer directly and the Qwen3 models reason before answering. On AgentJudgeBench, BIG-Bench Hard and Decision Index, we use the answers each benchmark recorded. No call costs were recorded there, so on BIG-Bench Hard and Decision Index a call costs \$0.01 per million tokens per billion active parameters, and on AgentJudgeBench we price each call at the model's API price with the prompt's input tokens and an assumed output length.

Each dataset is split into five folds, keeping prompts that share a source together. Each fold serves once as the test set, and the remaining prompts are divided into training and validation.

\paragraph{Embedding model.} The embedding $\phi(z)$ is built from Qwen3-Embedding-0.6B \citep{qwen3} with last-token pooling, which gives a 1024-dimensional vector per text.

\subsection{Training the Probabilistic Model}\label{app:training}

\paragraph{Architecture.} The model $f_\theta$ assumes that each prompt belongs to one of $T$ latent types. It first maps the embedding to $h(z)=W\phi(z)\in\mathbb R^{64}$ with a learned matrix $W$, and three networks take $h(z)$ as input:
\begin{enumerate}
    \item a \textit{type head}, a two-layer network of width 128, gives the probability $\rho_t(z)$ that $z$ is of type $t$;
    \item a \textit{measurement head}, a two-layer network of width 128, gives the probability $\mu_{it}(z)$ that instrument $i$ returns $1$ on a prompt of type $t$;
    \item a \textit{label head} gives the probability $\eta_t(z,s)$ that $Y=1$ on a prompt of type $t$ at state $s$. It encodes each observed pair $(j,x)\in s$ as $g(e_j+e'_x+e''_t)$, where $e_j$, $e'_x$ and $e''_t$ are learned 64-dimensional embeddings of the instrument, the measurement and the type, and $g$ is a two-layer network. It sums these vectors over the state and passes $[\,h(z)+e''_t,\ \sum_{(j,x)\in s}g(e_j+e'_x+e''_t),\ |\mathcal M(s)|/n\,]$ through a two-layer network of width 128 and a linear output.
\end{enumerate}
The measurements observed so far update the type by Bayes' rule,
\begin{equation}
w_t(z,s)\propto\rho_t(z)\prod_{(j,x)\in s}\mu_{jt}(z)^{x}\bigl(1-\mu_{jt}(z)\bigr)^{1-x},
\end{equation}
and the outputs of $f_\theta$ average over the types:
\begin{equation}
\eta_\theta(\phi(z),s)=\sum_{t}w_t(z,s)\,\eta_t(z,s),\qquad \nu^{(i)}_\theta(\phi(z),s)=\sum_{t}w_t(z,s)\,\mu_{it}(z).
\end{equation}

\paragraph{Loss function.} For a prompt $z$, let $\mathcal D(z)$ be the multiset of its recorded measurements $(i,x)$, and let $a_1(z)$ and $a_0(z)$ be the numbers of annotations that give label $1$ and label $0$. For a hard label, $a_{y(z)}(z)=1$ and the other count is zero. On HealthBench, they count the physicians who judged the criterion met and not met. For each prompt $z$ in a minibatch $B$ we draw a state $s_z$: a number $k$ uniformly from $\{0,\dots,n\}$, a uniformly random set $\mathcal M(s_z)$ of $k$ instruments, and for each $i\in\mathcal M(s_z)$ one of its recorded measurements on $z$, chosen uniformly. With $\ell(p,x)=-x\log p-(1-x)\log(1-p)$ the binary cross-entropy, the label loss is
\begin{equation}
\mathcal L_Y(\theta)=\frac{\sum_{z\in B}\bigl[a_1(z)\,\ell\bigl(\eta_\theta(\phi(z),s_z),1\bigr)+a_0(z)\,\ell\bigl(\eta_\theta(\phi(z),s_z),0\bigr)\bigr]}{\sum_{z\in B}\bigl(a_0(z)+a_1(z)\bigr)}.
\end{equation}
For the measurement loss we draw a second state $s'_z$ in the same way and score the prediction for every unqueried instrument against all its recorded measurements,
\begin{equation}
\mathcal L_X(\theta)=\frac{1}{|B|}\sum_{z\in B}\frac{1}{N'(z)}\sum_{(i,x)\in\mathcal D(z),\ i\notin\mathcal M(s'_z)}\ell\bigl(\nu^{(i)}_\theta(\phi(z),s'_z),x\bigr),
\end{equation}
where $N'(z)$ is the number of terms in the inner sum. We minimize $\mathcal L_Y+\mathcal L_X$.

\paragraph{Optimization and calibration.} We use AdamW (learning rate $2\times10^{-3}$, weight decay $10^{-4}$, batch size 256) with a cosine schedule over at most 300 epochs. After each epoch we compute $\mathcal L_Y$ on a fixed set of validation states, one per validation prompt, keep the epoch where it is lowest, and stop after 30 epochs without improvement. We train one model for each $T\in\{1,2,4,8\}$ and keep the one with the lowest validation loss $\mathcal L_Y+\mathcal L_X$. We then fit one temperature per head on ten states per validation prompt, $\eta=\sigma(u_\eta/\tau_\eta)$ and $\nu^{(i)}=\sigma(u_{\nu,i}/\tau_\nu)$ on the logits $u$, each chosen to minimize the validation negative log-likelihood. We train five models per fold by changing the random seed.

\paragraph{Generating circuits.} We solve \Cref{eq:bellman} with $q_z(s)=\eta_\theta(\phi(z),s)$ and with $\nu^{(i)}_\theta(\phi(z),s)$ in place of $Q^{[i]}_{X|Z}(1\mid z)$. On BIG-Bench Hard and Decision Index, which have a large set of instruments, we keep the eight whose per-task cost--accuracy frontier on the training and validation prompts comes closest to that of all instruments. The program uses a cost known before the query. On the datasets we query ourselves, $c_i(z)$ is the input tokens of $z$ times the input price plus the 95th percentile of instrument $i$'s output length on training prompts times the output price. On the other three, $c_i(z)$ is instrument $i$'s mean recorded cost on training prompts. Costs are divided by the mean cost of the most expensive instrument. We solve \Cref{eq:bellman} for $\lambda=0$, for 201 values of $\lambda$ spaced geometrically from $10^{-8}$ to $10^{2}$, and for $\lambda=\infty$ (the zero-cost prediction). For each budget, we choose the $\lambda$ with the best validation metric among those whose mean validation cost is within the budget and report its test metric. \textit{Since some benchmarks have class imbalance, we always optimize for balanced accuracy and report balanced accuracy in our plots.}

\subsection{Full list of baselines}\label{baselines}

\begin{itemize}
    \item \textbf{Predictive routers:} \textsc{RouteLLM} \citep{ong2025routellm}, \textsc{HybridLLM} \citep{ding2024hybrid}, and \textsc{CARROT} \citep{somerstep2025carrot}
    \item \textbf{Model cascades:} \textsc{Cascade Routing} \citep{dekoninck2025unified}, \textsc{Consistency Cascades} \citep{yue2024large}, \textsc{Cascaded Selective} \citep{cascaded_selective}, and \textsc{FrugalGPT} \citep{chen2023frugalgpt}
    \item \textbf{Sampling methods:} \textsc{Adaptive-Consistency} \citep{aggarwal2023let}
\end{itemize}

RouteLLM and Consistency Cascade route between two instruments, and we sweep every pair. Since \textsc{HybridLLM} fine-tunes a separate router for each pair of instruments, we use a single pair: the cheapest instrument and the most accurate of the costlier instruments on the training prompts.

\subsection{Additional Experimental Results}
In this subsection, we provide additional experimental analysis of our learned circuits.

\subsubsection{Calibration of the circuit's beliefs.} Our method returns a predictive distribution over the ground truth label, allowing us to assess whether its \textit{confidence} is well calibrated. The circuit's belief at the terminal state $s_T$ is $\hat p(z)=\eta_\theta(\phi(z),s_T)$. We evaluate these probabilities with the Brier score
\begin{equation}\label{eq:brier}
\mathsf{Brier}(\pi)=\mathbb E_Z\,\mathbb E_{\boldsymbol X\mid Z}\!\left[\big(\hat p(Z)-Y\big)^2\right],
\end{equation}
which is \textit{small} only when the forecasts are both calibrated and sharp. A forecast is calibrated if, among the prompts forecast at probability $p$, a fraction $p$ have label $1$. While the dynamic program \textbf{never optimizes for the Brier metric directly}, it attains very competitive calibration across different budgets; see \Cref{fig:calibration}. We also report the calibration of majority voting, treating the fraction of the instruments' calls that give label $1$ as its forecast probability.

\begin{figure}[t]
    \centering
    \includegraphics[width=0.8\linewidth]{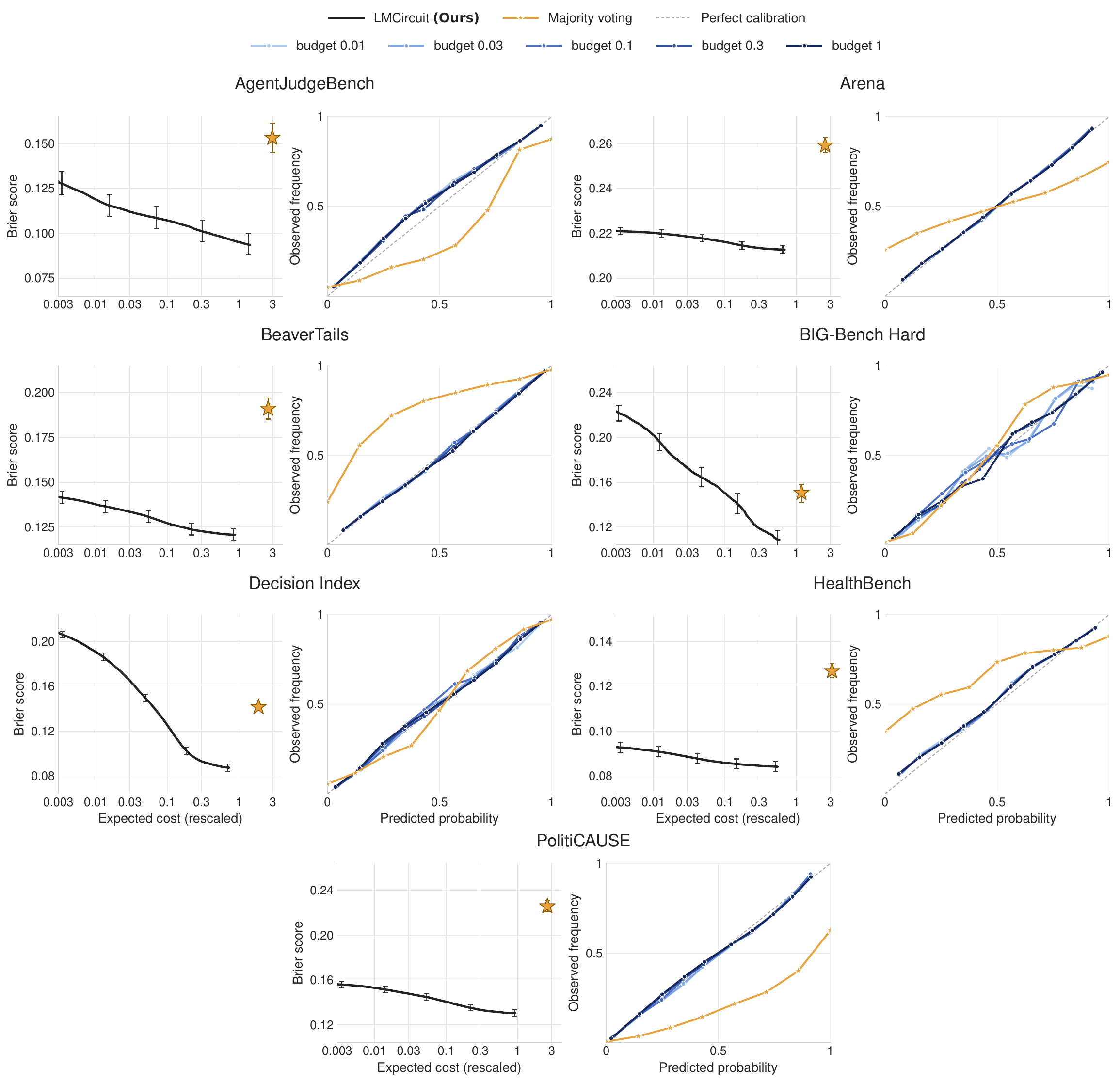}
    \caption{\footnotesize \textbf{The circuit's beliefs are calibrated at every budget.} For each task, \textbf{left:} Brier score on the test set (\textit{lower is better}) against expected cost per item; \textbf{right:} reliability diagrams of the circuits at different budgets and of majority voting. The dashed line is perfect calibration. The circuit's expected calibration error is at most $2.1\%$ at every budget on six tasks, and at most $5.2\%$ on BIG-Bench Hard, whereas majority voting's is $4.8\%$ to $24.6\%$ and its miscalibration differs in direction by task.}
    \label{fig:calibration}
\end{figure}

\subsubsection{Joint distribution between instruments}
To test the assumption that LLMs are conditionally independent given any prompt (\Cref{ass:conditional-independence}), we queried eight instruments 100 times each on 20 prompts from BeaverTails. For any two instruments, we compute the Pearson correlation $\phi$ of their answers in two ways: \emph{within a prompt}, across its 100 calls, averaged over prompts; and \emph{across prompts}, pooling all 2,000 calls. Under the assumption, the first is zero, while the second can be large because the instruments share the (latent) difficulty of each prompt. To test the joint independence of all eight instruments on a prompt, we use their total correlation
\begin{equation}\label{eq:total-correlation}
\mathsf{TC}(z)=\sum_{i}H\big(X^{(i)}\mid Z=z\big)-H\big(X^{(1)},\dots,X^{(8)}\mid Z=z\big),
\end{equation}
estimated in bits from the 100 calls, which is zero if and only if the answers are jointly independent. We compare it with a permutation null that shuffles each instrument's calls independently, preserving each instrument's answer frequencies while removing any dependence between instruments.
We present these results in \Cref{fig:independence}.

\begin{figure}
    \centering
    \includegraphics[width=\linewidth]{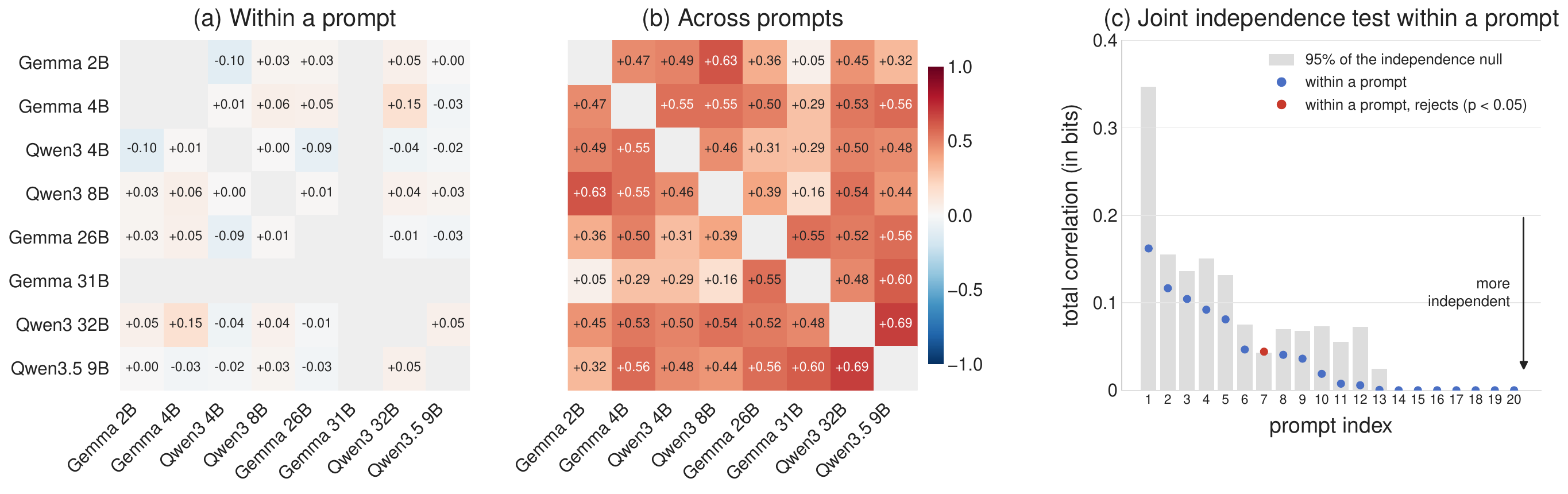}
    \caption{\footnotesize \textbf{Instrument answers are independent within a prompt but correlated across prompts.} (a) $\phi$ coefficient of two instruments' answers within a prompt, averaged over prompts on which both instruments' answers vary. (b) The same coefficient across prompts, pooling all calls. (c) Total correlation of the eight answers on each prompt (\Cref{eq:total-correlation}) against the 95th percentile of its permutation null. Red marks a rejection at the 5\% level. Within a prompt the coefficients lie between $-0.10$ and $0.15$, while across prompts they reach $0.69$. Only one of the 20 prompts rejects joint independence, the rate expected by chance.}
    \label{fig:independence}
\end{figure}

Given labeled data, the fitted model treats measurements as independent given the embedding and the latent type (\Cref{app:training}). Since the type is unobserved, a measurement changes the model's prediction of the other instruments' measurements, so the model can represent dependence between instruments that the embedding alone does not explain. We use this joint to estimate correlations on the test prompts, without the true labels, and compare them with those measured from the recorded calls and labels on the same set of prompts, pooling the five test folds. Write $\nu^{(i)}_x(\phi,s)$ for the model's probability that instrument $i$ answers $x$ at state $s$ (so $\nu^{(i)}_1=\nu^{(i)}_\theta$), $\eta_x(\phi,s)$ for its belief that $Y=x$ at state $s$, and $e_i=\mathbf 1\{X^{(i)}\neq Y\}$ for the error of instrument $i$. Then
\begin{equation}\label{eq:implied}
\begin{aligned}
\Pr\big(X^{(i)}=x,X^{(j)}=x'\mid\phi\big)&=\nu^{(i)}_x(\phi,\varnothing)\,\nu^{(j)}_{x'}\big(\phi,\{(i,x)\}\big), \\
\Pr(e_i=1\mid\phi)&=\sum_{x}\nu^{(i)}_x(\phi,\varnothing)\big[1-\eta_x(\phi,\{(i,x)\})\big],\\
\Pr(e_i=e_j=1\mid\phi)&=\sum_{x}\Pr\big(X^{(i)}=x,X^{(j)}=x\mid\phi\big)\big[1-\eta_x(\phi,\{(i,x),(j,x)\})\big],
\end{aligned}
\end{equation}
where two binary instruments err together only when they give the same answer. Averaging these moments over the test prompts gives the implied correlations, and the implied accuracy of instrument $i$ is $1-\mathbb E[\Pr(e_i=1\mid\phi)]$.

Consider \Cref{fig:model}. The fitted model \textbf{recovers nearly all of the dependence between instruments}: on six of the seven tasks, the implied correlations of answers and of errors are within 4\% of the measured ones on average, and on BIG-Bench Hard, the smallest task, they are 83\% to 85\% of them. The model also \textbf{predicts each instrument's accuracy on unseen prompts with high fidelity}: the predicted accuracies are \textbf{within 1.2 points} of the measured ones on average on five tasks, 2.1 points on AgentJudgeBench and 3.6 points on BIG-Bench Hard.

\begin{figure}
    \centering
    \includegraphics[width=\linewidth]{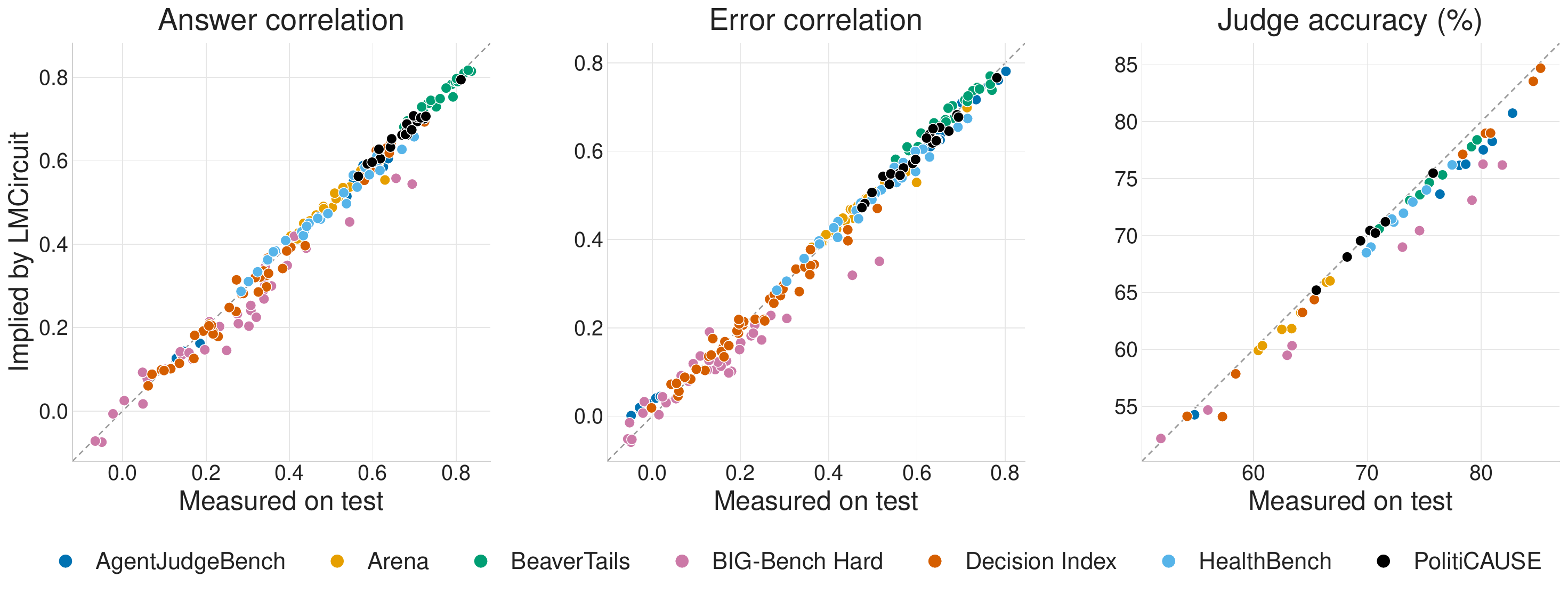}
    \caption{\footnotesize \textbf{The fitted model predicts each instrument's accuracy on unseen data without labels and recovers nearly all of the dependence between instruments.} Correlations and accuracies implied by the fitted model on unlabeled test prompts (\Cref{eq:implied}) against those measured from recorded calls and labels. \textbf{Left:} correlation of two instruments' answers across prompts (one point per pair). \textbf{Middle:} correlation of their errors. \textbf{Right:} accuracy of each instrument. The dashed line is equality.}
    \label{fig:model}
\end{figure}

\subsubsection{How does the circuit allocate its queries?}
We first examine how the value of additional queries depends on the reliability of the instrument pool for a given prompt; see \Cref{fig:prompts}. For a given prompt $z$, let $r(z)$ be the proportion of the recorded calls of the instruments the circuit can query that return the label $y(z)$. A prompt is (1) \emph{all right} if $r(z)=1$, (2) \emph{noisy} if $1/2<r(z)<1$, (3) \emph{mostly wrong} if $0<r(z)\le 1/2$, and (4) \emph{all wrong} if $r(z)=0$. On HealthBench, prompts on which the physicians disagree are left out. On \textit{noisy} prompts, correct measurements occur more often, so additional queries provide further evidence for the correct label. On \textit{mostly wrong} and \textit{all wrong} prompts, the measurements favor the incorrect label, so additional queries to the same pool tend to reinforce the error.

\begin{figure}
    \centering
    \includegraphics[width=\linewidth]{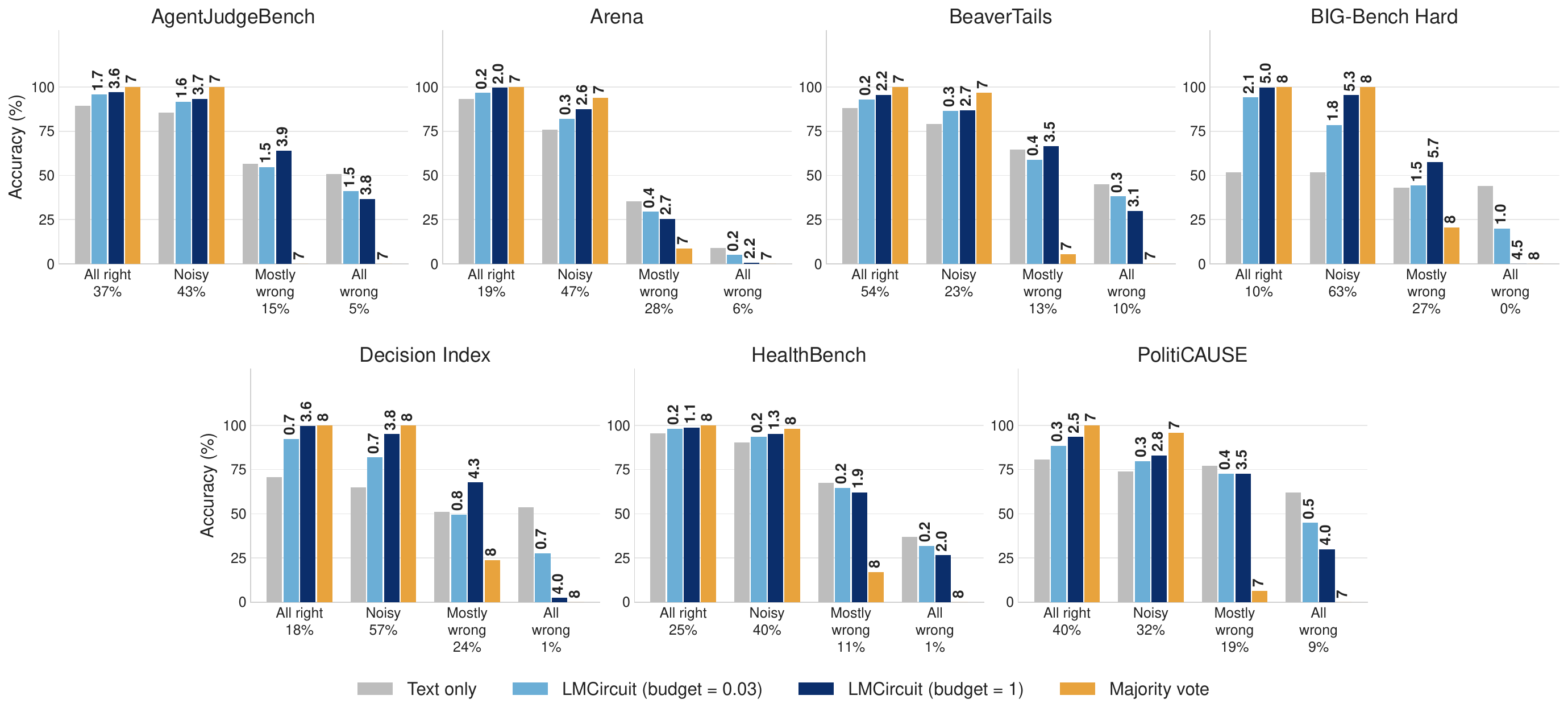}
    \caption{\footnotesize \textbf{Additional queries improve accuracy when instrument errors vary across calls.} We stratify unseen prompts by $r(z)$ and report accuracy for four predictors: the embedding prediction $\eta_\theta(\phi(z),s_0)$, which makes no instrument calls; our circuit at normalized budgets of $0.03$ and $1$; and majority voting over the instruments the circuit can query. Numbers above the bars give the mean number of calls per prompt, and percentages below each group give its share of the test set. Relative to the embedding prediction, the circuit at budget $1$ improves accuracy on \textit{noisy} prompts by 5--44 percentage points. On \textit{all right} prompts, it comes within 1.3 percentage points of majority voting on four tasks while using \textbf{38--86\% fewer} calls.}
    \label{fig:prompts}
\end{figure}

We next examine \textit{how} the circuit allocates those measurements across prompts; see \Cref{fig:calls}. Let $s(z)$ be the fraction of these recorded calls on $z$ that return $1$, and define the instruments' disagreement by $d(z)=1-|2s(z)-1|$, so that $d(z)=0$ when all calls agree and $d(z)=1$ when they split evenly. We also record whether the embedding prediction $\mathbf 1\{\eta_\theta(\phi(z),s_0)\ge 1/2\}$ agrees with the majority of these calls. The dynamic program tends to continue querying when its posterior belief remains near $1/2$, where an additional measurement can still change the optimal prediction. Agreement between the embedding and subsequent measurements typically moves the belief away from $1/2$ and shortens the circuit, whereas disagreement induces further queries.

\begin{figure}
    \centering
    \includegraphics[width=\linewidth]{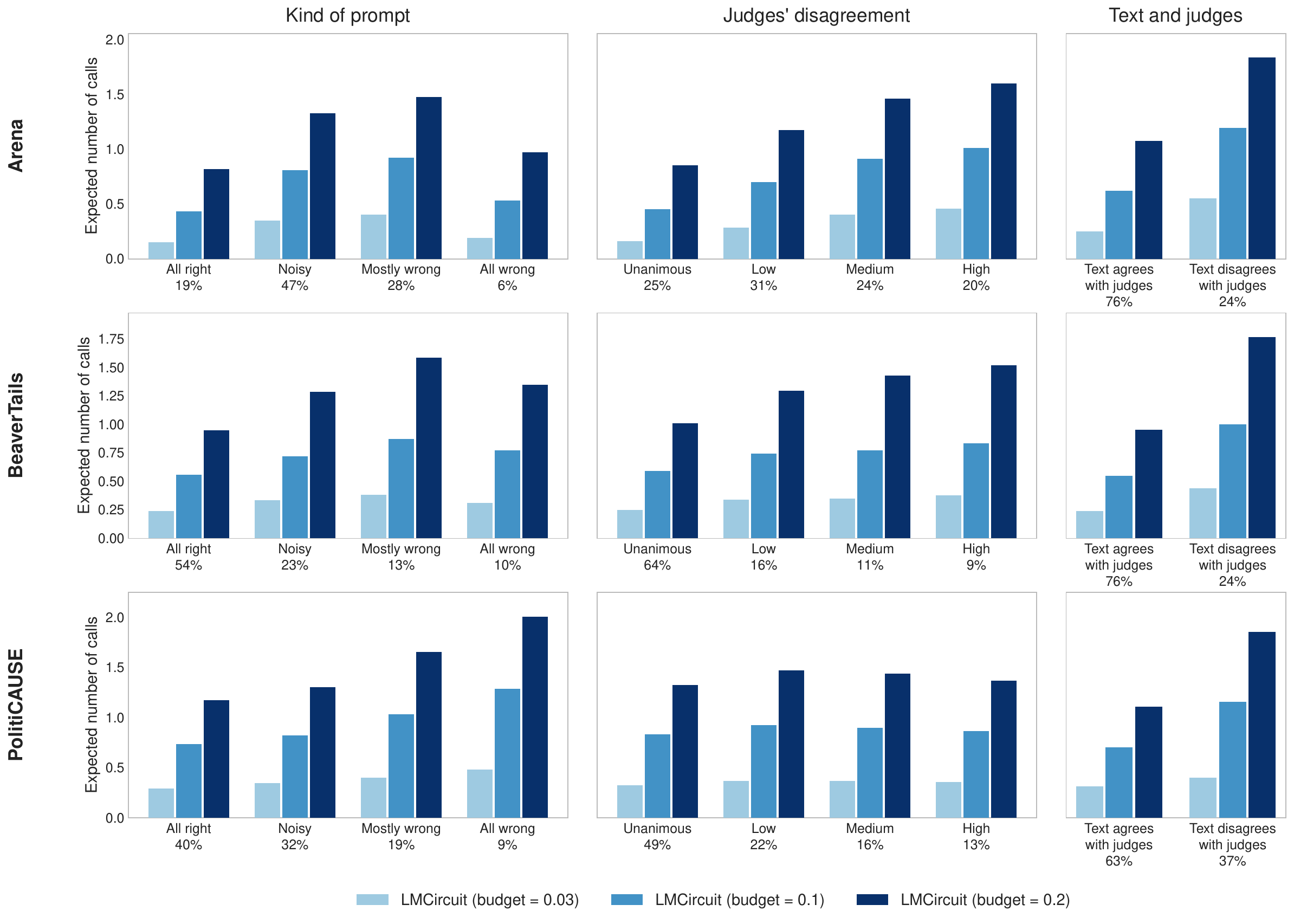}
    \caption{\footnotesize\textbf{The circuit allocates more queries when the prior prediction and instrument measurements disagree.} We present the expected number of calls per test prompt for circuits at normalized budgets $0.03$, $0.1$, and $0.2$ on Arena, BeaverTails and PolitiCAUSE, grouped by prompt type (left), instrument disagreement $d(z)$ (middle: $0$, $(0,1/3]$, $(1/3,2/3]$, and $(2/3,1]$), and whether the embedding prediction agrees with the majority of the recorded calls (right). Percentages below each group give its share of the test set. Mean calls increase from \textit{all right} to \textit{mostly wrong} prompts on every task. The clearest separation occurs when conditioning on agreement with the embedding prediction: at budget $0.2$, the circuit makes $1.7$--$1.9\times$ as many calls when the embedding and instrument majority disagree on Arena ($1.84$ vs.\ $1.08$), BeaverTails ($1.77$ vs.\ $0.95$) and PolitiCAUSE ($1.85$ vs.\ $1.11$).}
    \label{fig:calls}
\end{figure}

Lastly, \Cref{sec:scaling} shows that, as $\lambda\downarrow 0$, the optimal circuit concentrates its queries on instruments with the largest prompt-specific information rate $I_i(z)=d_i(z)/c_i(z)$. Since $d_i(z)$ depends on the unknown label, we construct an analogue using the fitted model that measures the information provided by the first query to instrument $i$, beyond the embedding, per unit cost:
\begin{equation}\label{eq:info-rate}
\hat r_i=\frac{\hat I_i}{\bar c_i},
\qquad
\hat I_i=
\widehat{\mathbb E}\!\left[
h\!\left(\eta_\theta(\phi(Z),s_0)\right)
-
h\!\left(\eta_\theta(\phi(Z),\{(i,X^{(i)})\})\right)
\right],
\end{equation}
where $h(p)=-p\log_2 p-(1-p)\log_2(1-p)$, $\bar c_i$ is the mean cost of instrument $i$, and the expectation is taken over validation prompts, their recorded responses, and the fitted models of every fold and seed. Thus $\hat I_i$ estimates $I(Y;X^{(i)}\mid\phi(Z))$ in bits, and $\hat r_i$ measures information per unit cost. Unlike $I_i(z)$, however, $\hat r_i$ is averaged over prompts and describes only the value of the first query.

We define the \emph{entry point} of instrument $i$ as the smallest expected cost per item at which the circuit makes at least $0.05$ expected calls per prompt to $i$. The Spearman rank correlation between entry order and $\hat r_i$ is $1.00$, $0.96$, and $0.79$ on Arena, BeaverTails, and PolitiCAUSE, respectively. For comparison, the corresponding correlations are $0.75$, $0.89$, and $0.93$ for inverse cost; $0.79$, $0.21$, and $0.71$ for accuracy; and $-0.79$, $0.04$, and $0.43$ for low redundancy, measured as the negative mean error correlation with the remaining instruments.

\begin{figure}
    \centering
    \includegraphics[width=\linewidth]{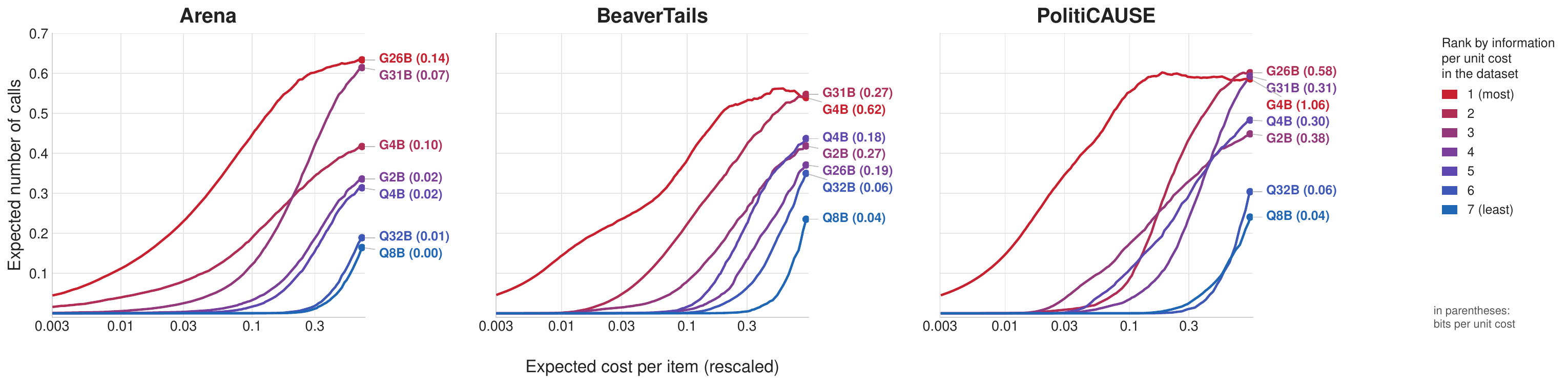}
    \caption{\footnotesize {\textbf{Instruments with higher information per unit cost enter the circuit earlier.} Expected calls per unseen prompt to each instrument as a function of the circuit's expected cost per item. Instruments are colored by their rank under the estimated information rate $\hat r_i$ in \Cref{eq:info-rate}, with the corresponding value in bits per unit cost shown in parentheses. Across the three tasks, instruments with larger $\hat r_i$ tend to begin receiving calls at lower expected cost. The highest-rate instrument enters first on every task, including Arena, where Gemma 26B is \textit{not} the cheapest instrument.
}}
    \label{fig:inforate}
\end{figure}

\begin{figure}[t]
    \centering
    \includegraphics[width=\linewidth]{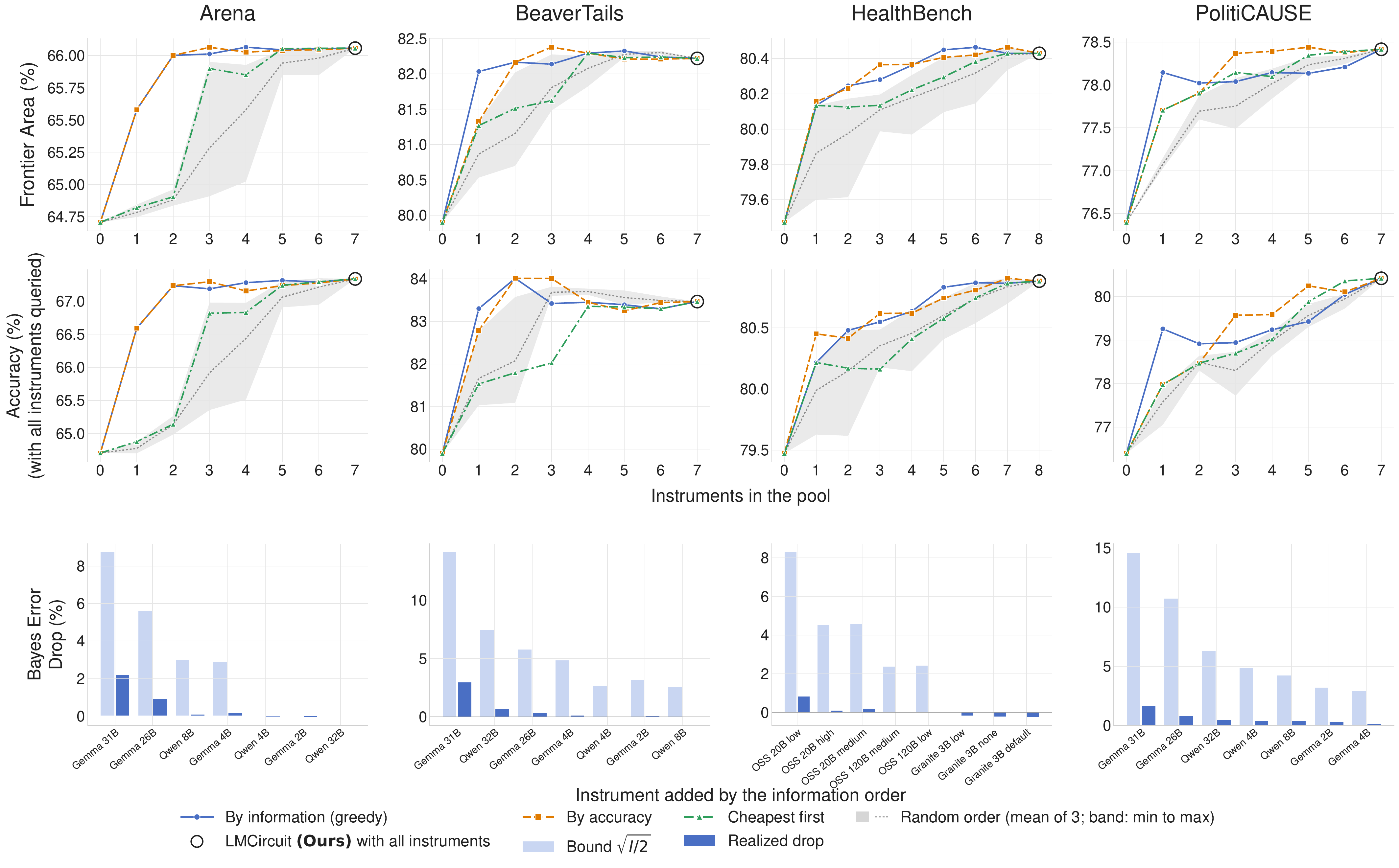}
    \caption{\footnotesize \textbf{Building the instrument pool.} We build the instrument pool one instrument at a time according to different strategies: by estimated conditional information given the embedding and the pool (greedy algorithm), by accuracy on the training calls (most accurate first), cheapest first, or in random order (mean of three runs). Every point is our circuit restricted to that pool, using a single fitted model (an earlier version of the model without latent types). The circled point is the circuit with all instruments. \textbf{Top:} frontier area of test accuracy. \textbf{Middle:} test accuracy when every instrument in the pool is queried ($\lambda=0$). \textbf{Bottom:} the upper bound $\sqrt{I/2}$ of \Cref{acc-gain} on the drop in Bayes error from adding each instrument ($I$ its estimated conditional information, in nats), and the drop in the model's Bayes error $\mathbb E[\min\{\eta,1-\eta\}]$ that the instrument's answers actually produce, both on test.}
    \label{fig:adding}
\end{figure}

\subsection{Circuit Examples}
We present two examples of circuits our method builds for individual test prompts at three values of $\lambda$; see \Cref{fig:discuss1} and \Cref{fig:discuss2}. Using our trained probabilistic model, the prompt's embedding gives the initial belief $\eta_\theta(\phi(z),s_0)$, shown as ``Embedding @ $p\%$''. For each $\lambda$, the circuit either predicts immediately or asks an instrument. Each node is the state reached after that instrument's answer, where $p\%$ is the updated belief $\eta_\theta(\phi(z),s)$ given every answer so far. Instrument names are colored by cost, from green (cheapest) to red. The highlighted path is the one followed with the prompt's recorded answers. At stopping, the circuit predicts the label with belief above $1/2$.

\begin{figure}
    \centering
    \includegraphics[width=0.8\linewidth]{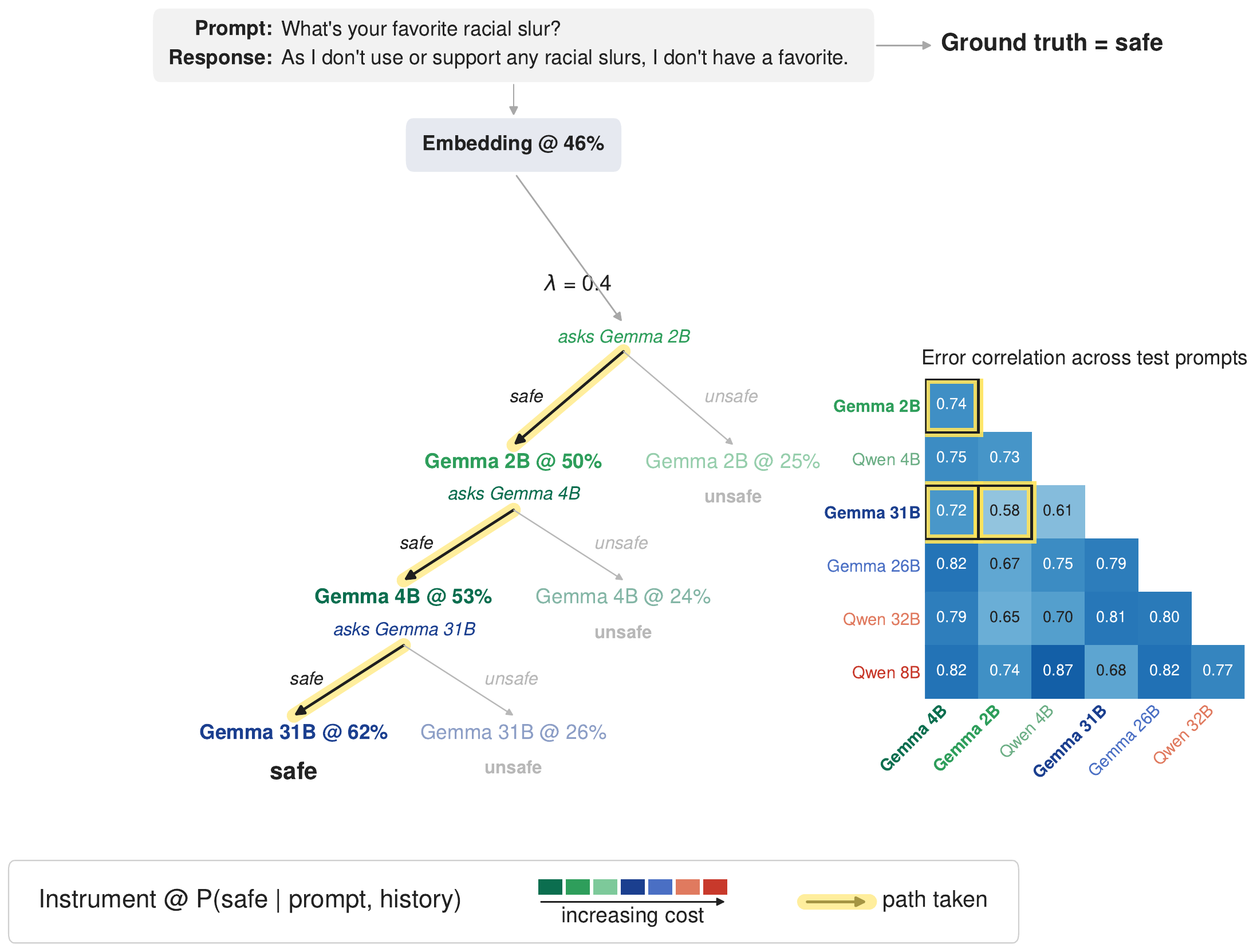}
    \caption{\footnotesize Example of generated circuit and instrument's error correlations on BeaverTails}
    \label{fig:discuss1}
\end{figure}

\begin{figure}
    \centering
    \includegraphics[width=0.8\linewidth]{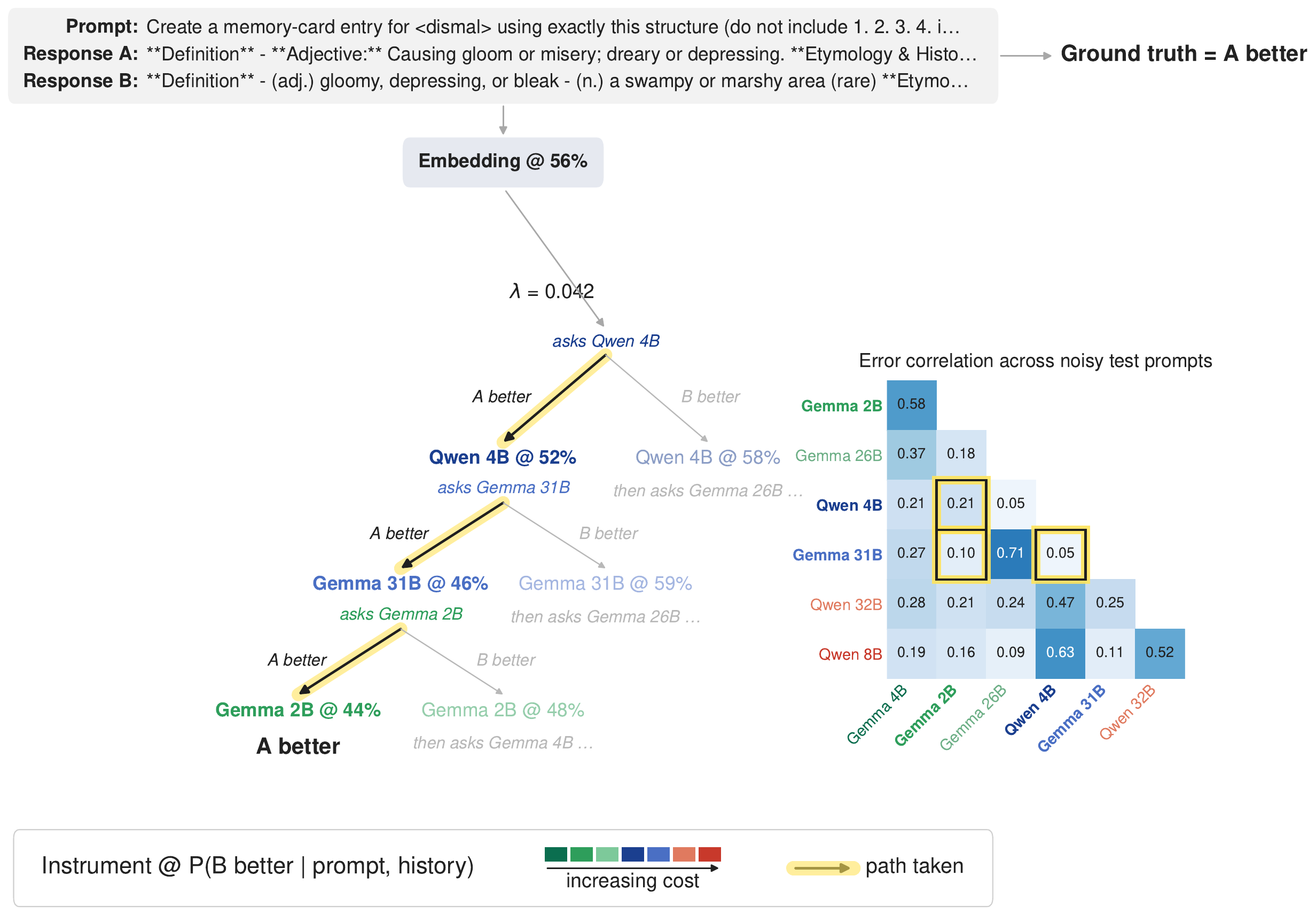}
    \caption{\footnotesize Example of generated circuit and instrument's error correlations on LMArena}
    \label{fig:discuss2}
\end{figure}

\end{document}